\documentclass{article}

\usepackage{microtype}
\usepackage{graphicx}
\usepackage{subcaption}
\usepackage{booktabs} 

\usepackage{hyperref}

\usepackage[preprint]{icml2026}

\usepackage{amsmath}
\usepackage{amssymb}
\usepackage{mathtools}
\usepackage{amsthm}

\usepackage[capitalize,noabbrev]{cleveref}
\usepackage{xcolor}
\definecolor{Forest_green}{rgb}{0.0, 0.5, 0.0}
\newcommand{\gd}[1]{\textcolor{Forest_green}{$\uparrow$#1}}

\theoremstyle{plain}
\newtheorem{theorem}{Theorem}[section]
\newtheorem{proposition}[theorem]{Proposition}

\theoremstyle{definition}

\theoremstyle{remark}

\usepackage{multirow}
\usepackage{multicol}

\usepackage{algorithm}
\usepackage{algorithmic}
\renewcommand{\algorithmiccomment}[1]{\hfill\# #1}
\providecommand{\RETURN}{\STATE \textbf{return}}

\usepackage[textsize=tiny]{todonotes}

\icmltitlerunning{MonoLoss: A Training Objective for Interpretable Monosemantic Representations}

\begin{document}

\twocolumn[
  \icmltitle{MonoLoss: A Training Objective for \\Interpretable Monosemantic Representations}



  \icmlsetsymbol{equal}{*}

  \begin{icmlauthorlist}
    \icmlauthor{Ali Nasiri-Sarvi}{conc}
    \icmlauthor{Anh Tien Nguyen}{stony}
    \icmlauthor{Hassan Rivaz}{conc-elec}
    \icmlauthor{Dimitris Samaras}{stony}
    \icmlauthor{Mahdi S. Hosseini}{conc,mila}
  \end{icmlauthorlist}

  \icmlaffiliation{conc}{Department of Computer Science and Software Engineering, Concordia University, Canada}
  \icmlaffiliation{stony}{Department of Computer Science, Stony Brook University, United States}
  \icmlaffiliation{conc-elec}{Department of Electrical and Computer Engineering, Concordia University, Canada}
  \icmlaffiliation{mila}{Mila--Quebec AI Institute, Canada}

  \icmlcorrespondingauthor{Mahdi S. Hosseini}{mahdi.hosseini@concordia.ca}


  \vskip 0.3in
]



\printAffiliationsAndNotice{}  

\begin{abstract}
  Sparse autoencoders (SAEs) decompose polysemantic neural representations, where neurons respond to multiple unrelated concepts, into monosemantic features that capture single, interpretable concepts. However, standard training objectives only weakly encourage this decomposition, and existing monosemanticity metrics require pairwise comparisons across all dataset samples, making them inefficient during training and evaluation. We study a recent MonoScore metric and derive a single-pass algorithm that computes exactly the same quantity, but with a cost that grows linearly, rather than quadratically, with the number of dataset images. On OpenImagesV7, we achieve up to a $1200\times$ speedup wall-clock speedup in evaluation and $159\times$ during training, while adding only $\sim$4\% per-epoch overhead. This allows us to treat MonoScore as a training signal: we introduce the Monosemanticity Loss (MonoLoss), a plug-in objective that directly rewards semantically consistent activations for learning interpretable monosemantic representations. Across SAEs trained on CLIP, SigLIP2, and pretrained ViT features, using BatchTopK, TopK, and JumpReLU SAEs, MonoLoss increases MonoScore for most latents. MonoLoss also consistently improves class purity (the fraction of a latent's activating images belonging to its dominant class) across all encoder and SAE combinations, with the largest gain raising baseline purity from 0.152 to 0.723. Used as an auxiliary regularizer during ResNet-50 and CLIP-ViT-B/32 finetuning, MonoLoss yields up to 0.6\% accuracy gains on ImageNet-1K and monosemantic activating patterns on standard benchmark datasets. The code is publicly available at \url{https://github.com/AtlasAnalyticsLab/MonoLoss}.
\end{abstract}

\section{Introduction}
\label{sec:intro}
    \begin{figure}[ht]
        \centering
        \includegraphics[width=\linewidth]{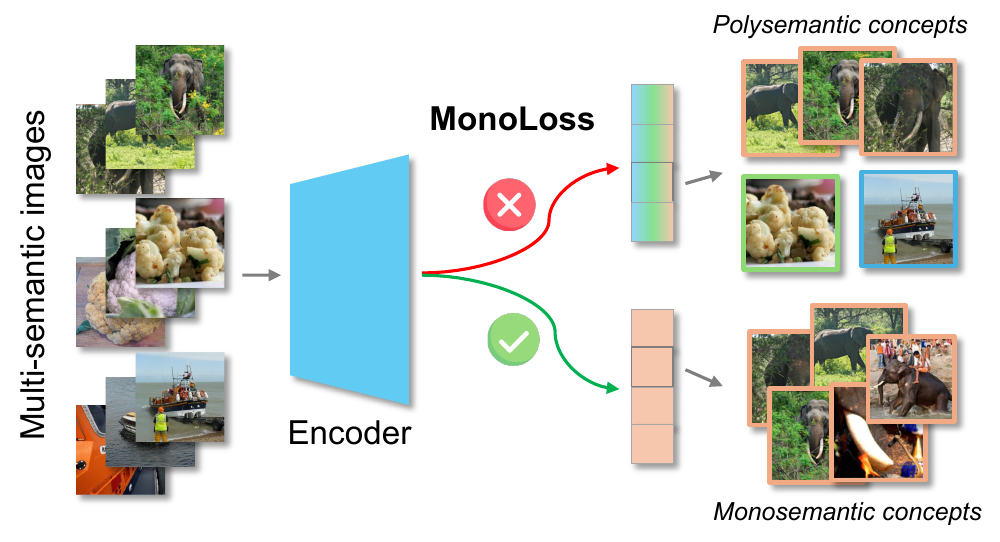}
        \caption{
        From multi-semantic images, an encoder produces latent feature dimensions. \textbf{\textcolor{red}{Without}} MonoLoss, a single dimension tends to be polysemantic, activating for disparate concepts. \textbf{\textcolor{olive}{With}} MonoLoss, dimensions consistently response to one concept, yielding coherent and interpretable monosemantic features.
        }
        \label{fig:monoloss}
    \end{figure}
    Pre-trained vision models are widely used as default feature extractors for images, yet the embedding process remains largely opaque, limiting interpretability \cite{xai1,xai2,xai3,xai4}. A key reason is polysemanticity: many units encode multiple, often unrelated concepts (e.g., human, cat, mountain) rather than a single, coherent feature \cite{olah2020zoom,kopf2025capturing,o2023disentangling}. This polysemantic behavior makes it difficult to analyze and explain how such models represent images~\cite{elhage2022toymodelssuperposition,gurnee2023finding,makelov2025towards}. 
    
    Although monosemantic representations are desirable, they are difficult to obtain under standard training regimes. Sparse Autoencoders (SAE) have recently emerged as a mechanistic interpretability approach \cite{bereska2024mechanistic,mudide2025efficient,Residual} that empirically decomposes polysemantic features into monosemantic latent units with unified semantics \cite{cunningham2023,bussmann2025matryoshka,yao2025adaptivesae}. However, common SAE objectives, typically reconstruction losses with sparsity constraints, promote monosemanticity indirectly, yielding representations that are still partially entangled. 
    Furthermore, pre-trained vision backbones are optimized primarily for downstream tasks rather than interpretability, providing no explicit constraints toward monosemanticity.
    
    Recently, the Monosemanticity Score (MonoScore)~\cite{monoscore} was introduced to quantify the monosemanticity of individual embedding dimensions. 
    It measures \textit{how similar are images that cause a particular neuron/dimension to activate}. If the unit activates for many unrelated concepts, then the corresponding score would be lower, whereas consistent activation for a single coherent concept yields a higher MonoScore.
    Despite this promise, MonoScore has been primarily used as a post-hoc evaluation metric rather than as a training signal, so it cannot shape representations during learning. Even if one wanted to use it during training, it would be prohibitively expensive since the original formulation scales quadratically with the number of images $O(N^2)$. This is particularly problematic because the reliable MonoScore evaluation requires larger evaluation sets. 
    
    Here, we reformulate MonoScore to run in linear time $O(N)$ w.r.t. the number of images, so that large-scale evaluation is efficient enough to use during training. Instead of explicitly computing all $N(N-1)/2$ pairwise similarities, we show that MonoScore can be rewritten in terms of simple per-image statistics accumulated in a single pass, 
    yielding an exact equivalent of the original score at linear cost.
    
    We further introduce MonoLoss, an objective function that directly optimizes monosemanticity. We instantiate MonoLoss in two settings: (i) SAEs, where it augments the reconstruction objective and sparsity constraints to increase monosemanticity of the interpretable latents, and (ii) fine-tuning of pre-trained vision backbones, where it serves as an auxiliary regularizer to obtain higher downstream task performance. Figure \ref{fig:monoloss} illustrates the concept coherence of activation latents obtained by MonoLoss after the optimization process. In summary, our contributions are: 
    \begin{itemize}
        \item \textbf{Linear-time MonoScore.} We convert MonoScore into a single-pass, equivalent, linear-time computation 
        with a massive speedup for large-scale dataset evaluation.
        \item \textbf{MonoLoss objective.} We introduce MonoLoss, a simple, plug-and-play objective that brings explicit  monosemantic representations into standard training.
        
        \item \textbf{SAE integration evaluation.} We integrate MonoLoss into SAE training, yielding consistent gains in monosemanticity and class purity across all encoder and SAE combinations.
        
        \item \textbf{Fine-tuning gains at near-zero cost for vision encoders.} We adopt MonoLoss to fine-tune pre-trained networks, achieving higher accuracy on ImageNet-1K, CIFAR-10, and CIFAR-100 with minimal computational overhead.
    \end{itemize}

\section{Related Works}
    \label{sec:related}
    \textbf{Sparse Autoencoders.} SAEs revisit classical dictionary learning to factor model activations into sparse linear codes whose atoms can be inspected and manipulated \cite{kreutz2003dictionary, mairal2009online, mairal2008supervised, OLSHAUSEN19973311}, and they have become a tool for mechanistic interpretability in both language and vision domains~\cite{cunningham2023,topksae, bricken2023monosemanticity}. Vanilla ReLU SAE~\cite{reluSAE}, TopK SAE~\cite{topksae}, BatchTopK SAE~\cite{batchtopksae}, and JumpReLU SAE~\cite{jumprelu} have been used to obtain monosemantic features, attribute model behavior to human-understandable concepts, and enable targeted editing or steering without architectural changes \cite{scalingsae,marks2025sparse,lieberum2024gemmascopeopensparse}. Evaluation typically considers sparsity, activation value, concept coverage, and explicit monosemanticity metrics \cite{lim2025sparse,monoscore}. Our work builds on this line by treating MonoScore not only as a diagnostic but also as an optimization target, introducing a novel loss that integrates seamlessly into training.
    
    \textbf{Monosemanticity measurement.} Monosemanticity quantifies the extent to which a latent corresponds to coherent concepts. Despite its importance for interpretability, formal metrics to measure monosemanticity remain scarce. Currently, only two prior works explicitly target this goal. In vision,~\cite{monoscore} proposes a Monosemanticity Score to evaluate SAE features, while in language,~\cite{conditional_loss} introduces the condition loss to encourage monosemanticity during training. Our work differs in both scope and methodology: we focus on the vision domain and treat monosemanticity as a training objective. 

\section{Methodology}

\subsection{Preliminary: Monosemanticity Score}

\begin{figure*}[htp]
\centering
\begin{minipage}[t]{0.48\linewidth}
\begin{algorithm}[H]
\caption{MonoScore (pairwise) - Baseline}
\label{alg:monoscore-baseline}
\begin{algorithmic}[1]
\STATE \textbf{Input:} Images $x_1,\dots,x_N$, encoder $E(\cdot)$, activations $a_{kn}$ for $k=1\ldots M$
\STATE \textbf{Output:} MonoScore $MS \in \mathbb{R}^M$
\STATE $h_n \gets E(x_n) / \|E(x_n)\|_2$ \COMMENT{$h_n \in \mathbb{R}^d$}
\STATE Min--max normalize $a_{kn}$ over $n$ to obtain $\tilde a_{kn} \in [0,1]$
\STATE Initialize $\text{num}, \text{den} \in \mathbb{R}^M$ to zero
\FOR{$n = 1$ to $N$}
  \FOR{$m = n+1$ to $N$}
    \STATE $s_{nm} \gets h_n^\top h_m$ \COMMENT{$s_{nm} \in \mathbb{R}$}
    \STATE $r_{:} \gets \tilde a_{:n} \odot \tilde a_{:m}$ \COMMENT{$r_{:} \in \mathbb{R}^M$}
    \STATE $\text{num} \gets \text{num} + r_{:} \cdot s_{nm}$ \COMMENT{$\mathbb{R}^M$}
    \STATE $\text{den} \gets \text{den} + r_{:}$ \COMMENT{$\mathbb{R}^M$}
  \ENDFOR
\ENDFOR

\STATE $MS_k \gets \begin{cases}
    \text{num}_k / \text{den}_k, & \text{if } \text{den}_k > 0 \\
    0, & \text{otherwise}
  \end{cases} \quad \forall k$
\RETURN MS  \COMMENT{$MS \in \mathbb{R}^M$}
\end{algorithmic}
\end{algorithm}
\end{minipage}\hfill
\begin{minipage}[t]{0.48\linewidth}
\begin{algorithm}[H]
\caption{MonoScore (linear-time in $N$) - Ours}
\label{alg:monoscore-linear}
\begin{algorithmic}[1]
\STATE \textbf{Input:} Images $x_1,\dots,x_N$, encoder $E(\cdot)$, activations $a_{kn}$ for $k=1\ldots M$
\STATE \textbf{Output:} MonoScore $MS \in \mathbb{R}^M$
\STATE $h_n \gets E(x_n) / \|E(x_n)\|_2$ \COMMENT{$h_n \in \mathbb{R}^d$}
\STATE Min--max normalize $a_{kn}$ over $n$ to obtain $\tilde a_{kn} \in [0,1]$
\STATE Initialize $u,v \in \mathbb{R}^M,\ w \in \mathbb{R}^{d \times M}$ to zero

\FOR{$n = 1$ to $N$}
  \STATE $u \gets u + \tilde a_{:n}$ \COMMENT{$u \in \mathbb{R}^M$}
  \STATE $v \gets v + \tilde a_{:n} \odot \tilde a_{:n}$ \COMMENT{$v \in \mathbb{R}^M$}
  \STATE $w \gets w + h_n\,\tilde a_{:n}^\top$ \COMMENT{$w \in \mathbb{R}^{d \times M}$}
\ENDFOR
\STATE $q \gets \sum_d w_{d:} \odot w_{d:}$ \COMMENT{$q \in \mathbb{R}^M$, $q_k = \|w_{:k}\|_2^2$}
\STATE $\text{num} \gets \tfrac{1}{2}(q - v)$ \COMMENT{$\text{num} \in \mathbb{R}^M$}
\STATE $\text{den} \gets \tfrac{1}{2}(u \odot u - v)$ \COMMENT{$\text{den} \in \mathbb{R}^M$}
\STATE $MS_k \gets \begin{cases}
    \text{num}_k / \text{den}_k, & \text{if } \text{den}_k > 0 \\
    0, & \text{otherwise}
  \end{cases} \quad \forall k$ 
\RETURN MS \COMMENT{$MS \in \mathbb{R}^M$}
\end{algorithmic}
\end{algorithm}
\end{minipage}
\end{figure*}

We build on the Monosemanticity Score (MonoScore)~\cite{monoscore} to quantify how consistently a neuron responds to a single visual concept. Intuitively, a neuron is \emph{monosemantic} if the images that strongly activate it are semantically similar. It is \emph{polysemantic} if those images are diverse and unrelated. Originally validated with human judgment \cite{monoscore}, MonoScore has proven useful for analyzing sparse autoencoders and vision-language models. 

Formally, we define MonoScore using the following setup. Let $\mathcal{I} = \{x_n\}_{n=1}^N$ be a set of $N$ images. Let $E(\cdot)$ be a frozen image encoder that defines a semantic similarity space, producing $\ell_2$-normalized embeddings
$   h_n \;=\; \frac{E(x_n)}{\|E(x_n)\|_2} \in \mathbb{R}^d,$
and cosine similarities $s_{nm} = h_n^\top h_m \in [-1, 1]$ for all unordered pairs $(n,m)$ where $n < m$.

For each image $x_n$ and neuron $k \in \{1,\dots,M\}$, let $a_{kn}$ denote the activation of neuron $k$ on $x_n$ in some target representations (e.g., an SAE latent space). MonoScore is agnostic to how the activations are obtained. It only has access to the activation $a_{kn}$ and the similarity score $s_{nm}$.

For comparability between the neurons, MonoScore first min--max normalizes activations
\begin{equation}
    \tilde a_{kn}
    \;=\;
    \frac{a_{kn} - \min_{n'} a_{kn'}}{\max_{n'} a_{kn'} - \min_{n'} a_{kn'}}
    \;\in [0,1],
\end{equation}
which rescales each neuron's activations to its dynamic range within the dataset, enabling comparison across neurons. The MonoScore for a neuron $k$ is then an activation-weighted average of pairwise similarities:
\begin{equation}
    MS_k
    \;=\;
    \frac{\displaystyle \sum_{n < m} \tilde a_{kn} \tilde a_{km} \, s_{nm}}
         {\displaystyle \sum_{n < m} \tilde a_{kn} \tilde a_{km}}.
    \label{eq:ms_def}
\end{equation}
The term $\tilde a_{kn} \tilde a_{km}$ can be read as a \emph{co-activation weight} for the pair $(x_n, x_m)$: it is large only when the neuron $k$ fires strongly on both images and small otherwise. The numerator thus aggregates cosine similarities $s_{nm}$ with higher weight on pairs where $k$ is jointly active, while the denominator is the total co-activation mass over all pairs. Their ratio is thus the average similarity, in the embedding space of $E(\cdot)$, between images that co-activate neuron $k$: when $MS_k$ is high, the images that produce strong activations from $k$ lie close together and $k$ behaves like a coherent concept detector, whereas low $MS_k$ means that the images strongly activating k are semantically diverse and do not form a coherent concept. Algorithm \ref{alg:monoscore-baseline}, proposed by \cite{monoscore}, implements Equation \eqref{eq:ms_def} directly. It iterates over all $O(N^2)$ image pairs, computing cosine similarities $s_{nm}$ and updating the neuron-wise numerator and denominator vectors. This pairwise approach is faithful to the definition, but its quadratic scaling in the number of images makes it prohibitively slow for repeated evaluation during training. The next subsection derives an equivalent formulation eliminating this bottleneck.

\subsection{Linear MonoScore Algorithm}
Equation \eqref{eq:ms_def} defines $MS_k$ through pairwise comparisons, which is intuitive but requires iterating over all $O(N^2)$ image pairs, as in Algorithm~\ref{alg:monoscore-baseline}. This quickly becomes the dominant cost when MonoScore is evaluated on large datasets. We thus seek an equivalent formulation whose computation scales linearly in the number of images.

We derive a linear-time algorithm for MonoScore by rewriting both the numerator and denominator of \eqref{eq:ms_def} in terms of simple per-neuron and per-image statistics. Instead of summing directly over all pairs $(n,m)$, Algorithm~\ref{alg:monoscore-linear} maintains three quantities for each neuron $k$ in a single pass: the sum of normalized activations $u_k = \sum_n \tilde a_{kn}$ (accumulated in vector $u \in \mathbb{R}^M$), the sum of squared activations $v_k = \sum_n \tilde a_{kn}^2$ (accumulated in $v \in \mathbb{R}^M$), and an embedding-weighted sum $w_{:k} = \sum_n \tilde a_{kn} h_n$ (accumulated as columns of $w \in \mathbb{R}^{d \times M}$). 

From these statistics, the algorithm computes expressions
\[
    \text{num}_k = \tfrac{1}{2}\bigl(\|w_{:k}\|_2^2 - v_k\bigr),
    \qquad
    \text{den}_k = \tfrac{1}{2}\bigl(u_k^2 - v_k\bigr),
\]
and returns $MS_k = \text{num}_k / \text{den}_k$ whenever $\text{den}_k > 0$.


\begin{proposition}[Equivalence of MonoScore algorithms]
\label{prop:monoscore-equivalence}
Let $MS^{\text{pair}} \in \mathbb{R}^M$ denote the MonoScore computed by
Algorithm~\ref{alg:monoscore-baseline} and $MS^{\text{lin}} \in \mathbb{R}^M$
the MonoScore computed by Algorithm~\ref{alg:monoscore-linear}.
Then for every neuron $k \in \{1,\dots,M\}$,
\[
    MS^{\text{pair}}_k = MS^{\text{lin}}_k.
\]
\end{proposition}


The proof, provided in the Appendix \ref{app:proof-monoscore-equivalence}, expands the double sums in \eqref{eq:ms_def} using combinatorial identities. Algorithm~\ref{alg:monoscore-linear} computes exactly the same MonoScore scores as the original pairwise procedure, but with a single pass over $N$ images and no explicit pairwise similarities.

\subsection{MonoLoss: From Metric to Training Objective}

The MonoScore metric operates as a post-hoc evaluation: given a dataset of $N$ images, Algorithms~\ref{alg:monoscore-baseline} and~\ref{alg:monoscore-linear} quantify neuron monosemanticity. To optimize sparse autoencoders for monosemanticity, we reformulate this metric as a differentiable loss computable at each gradient step.

During training, we approximate MonoScore \emph{per batch}. Given a mini-batch of size $B$, we compute a batch-wise MonoScore $MS_k^{\text{(batch)}}$ for each neuron $k$ by applying the same linear-time algorithm as in Algorithm~\ref{alg:monoscore-linear}, but with $N$ replaced by $B$ (i.e., all sums and statistics are taken over the current batch instead of the whole dataset). Neurons whose batch denominator is numerically zero are treated as inactive and excluded from the average. 

Let $\mathcal{K}_{\text{active}}$ be the set of neurons with a valid batch score. Our \textbf{MonoLoss} is simply
\[
    \mathcal{L}_{\text{mono}}
    \;\triangleq\;
    1 - \frac{1}{|\mathcal{K}_{\text{active}}|}
        \sum_{k \in \mathcal{K}_{\text{active}}}
        MS_k^{\text{(batch)}},
\]
with $\mathcal{L}_{\text{mono}} = 1$ if no neuron is active. Minimizing $\mathcal{L}_{\text{mono}}$ increases the average batch-wise MonoScore, encouraging neurons to respond to semantically consistent sets of images

The overall training objective simply augments the base loss with our regularizer:

\begin{equation}
    \mathcal{L}
    \;=\;
    \mathcal{L}_{\text{base}}
    \;+\;
    \lambda \,\mathcal{L}_{\text{mono}} .
    \label{eq:MonoLoss}
\end{equation}
where $\mathcal{L}_{\text{base}}$ denotes the original task loss (e.g., reconstruction error, sparsity penalties, cross entropy), and $\lambda$ controls the strength of the monosemanticity pressure. For reporting and ablations, we compute full-dataset MonoScores using the linear algorithm and use the pairwise implementation as a reference to verify equivalence and measure speedup.

\subsection{Computational Complexity and Speedup}
\label{sec:speedup}

We analyze the computational cost of both MonoScore algorithms for a dataset
of $N$ images, $M$ neurons, and $d$-dimensional encoder features.

\textbf{Pairwise Algorithm Complexity.}
The baseline implementation, Algorithm~\ref{alg:monoscore-baseline}, iterates over all $\binom{N}{2} = \frac{N(N-1)}{2}$ unique image pairs. For each pair $(n,m)$, it performs four operations: one cosine similarity calculation, $s_{nm} = h_n^\top h_m$, one $M$-dimensional element-wise product, $r_{:} = \tilde a_{:n} \odot \tilde a_{:m}$, and two $M$-dimensional vector updates for the numerator and denominator accumulators. The cost per pair is therefore $O(d + M)$. The total computational complexity of the pairwise algorithm, $C_{\text{pair}}$, is:
\begin{equation*}
      C_{\text{pair}}(N, M, d) = \frac{N(N-1)}{2} \cdot O(d + M) = O(N^2(d+M))  
\end{equation*}
This quadratic scaling in the number of samples $N$ makes the algorithm intractable for large datasets.

\textbf{Linear-Time Algorithm Complexity.}
Our proposed implementation in Algorithm~\ref{alg:monoscore-linear} iterates through the $N$ images only once. For each image $n$, it updates three statistics: $u \in \mathbb{R}^M$, $v \in \mathbb{R}^M$, and $w \in \mathbb{R}^{d \times M}$. The costs for each image involve updates to $u$ and $v$, which are element-wise vector additions costing $O(M)$ each, and the update to $w$ via an outer product, $w \leftarrow w + h_n \tilde a_{:n}^\top$, which is the dominant step of $O(dM)$ operations. After the single loop, the algorithm computes the final numerators and denominators. This involves calculating $q_k = \|w_{:k}\|_2^2$ for all $M$ neurons, which costs $O(dM)$, and final element-wise operations costing $O(M)$. The total complexity $C_{\text{lin}}$ is dominated by the loop over $N$ images:
\begin{equation*}
    C_{\text{lin}}(N, M, d) = N \cdot O(dM + M) + O(dM) = O(NdM)
\end{equation*}
For fixed model dimensions $(d, M)$, our algorithm's complexity scales linearly $O(N)$ with the number of samples, in contrast to the baseline's $O(N^2)$ complexity.

\textbf{Implementation details.}
Both algorithms are written with explicit loops for clarity, but implemented using batched tensor operations. In the pairwise reference (Algorithm~\ref{alg:monoscore-baseline}), the outer loop over anchor images is kept while the inner loop processes blocks of $B_{\text{pair}}$ comparisons at once via vectorized cosine similarities and tensor reductions. In our linear-time version (Algorithm~\ref{alg:monoscore-linear}), the dataset is traversed in mini-batches of size $B$, updating $u$, $v$, and $w$ with matrix multiplications and vectorized additions.

\textbf{Empirical Wall-Clock Speedup.}
We empirically validate this theoretical analysis by benchmarking both algorithms, measuring the wall-clock time to compute MonoScore for datasets of geometrically increasing size $N$. All experiments are executed in PyTorch \cite{paszke2019pytorch} on a single NVIDIA H100 80GB GPU. Figure~\ref{fig:training_time_speedup} confirms our complexity analysis. The runtime of the baseline algorithm clearly exhibits the expected $O(N^2)$ quadratic scaling, with its slope appearing much steeper than our linear-time formulation, which scales in line with its $O(N)$ complexity.

\begin{figure}[ht]
    \centering
    \includegraphics[width=0.86\linewidth]{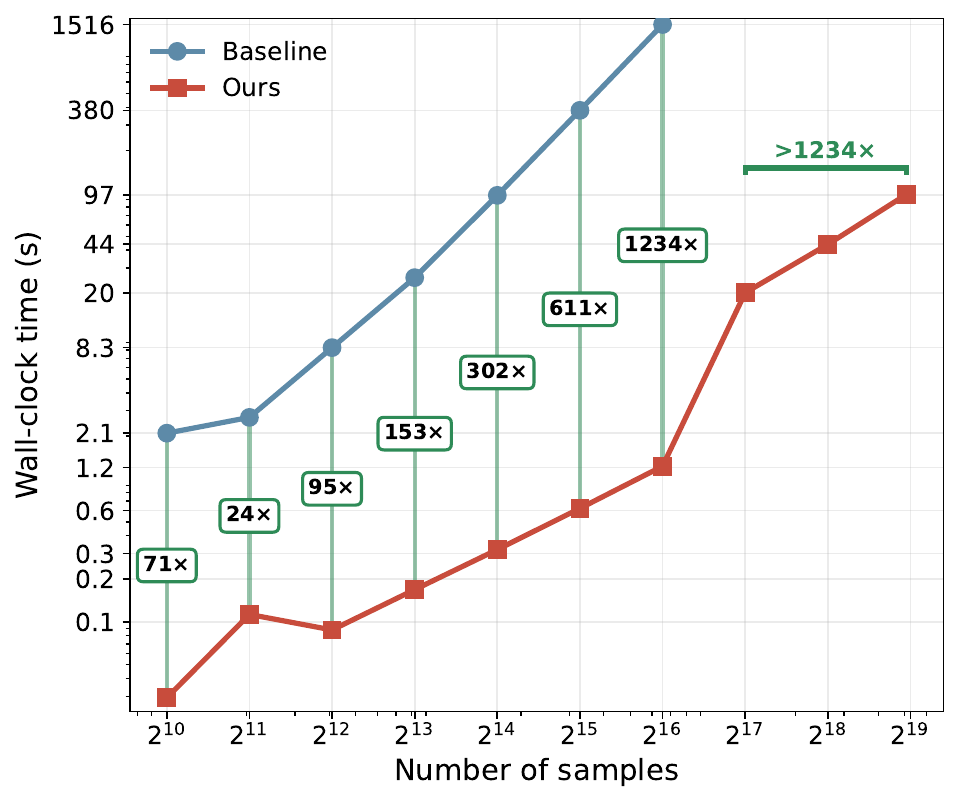}
    \caption{
        Wall-clock time to compute MonoScore as a function of dataset size $N$  (log--log scale), benchmarked on an NVIDIA H100 GPU. The baseline implementation (Algorithm~\ref{alg:monoscore-baseline}) clearly exhibits $O(N^2)$ quadratic scaling. In contrast, our linear-time      formulation (Algorithm~\ref{alg:monoscore-linear}) scales linearly ($O(N)$). Green annotations highlight the speedup factor, which reaches 1234$\times$ at $N=2^{16}$, the largest point evaluated for the prohibitively slow baseline.
    }
    \label{fig:training_time_speedup}
\end{figure}

While the baseline algorithm remains tractable for small $N$, such as a single training batch ($N \approx 2048$), its quadratic cost becomes prohibitive for evaluation on larger validation sets. At $N=2^{16}$ (65,536) samples, the baseline takes about 1516 seconds (over 25 minutes), whereas our method processes the same amount of data in only 1.2 seconds, which represents a 1234$\times$ speedup. The baseline evaluation was not extended beyond $N=2^{16}$ as its prohibitive quadratic runtime was already clearly established. In contrast, our linear-time method continues to scale efficiently, as shown by the measurements up to $N=2^{19}$. For a $50\text{k}$-sample dataset comparable to the ImageNet-1K validation set, the reference implementation takes $18.8\,\text{min}$, whereas our linear-time method takes $0.9\,\text{s}$ ($\sim1200\times$ speedup).

Using MonoLoss from Eq.~\ref{eq:MonoLoss} as an additional regularizer, we train a BatchTopK SAE on pre-computed CLIP-ViT-L/14 \cite{clip} features from OpenImagesV7 \cite{OpenImages} (batch size 2048, 8192 latents) with the same efficient feature loader and batching setup for all configurations. In this setting, a single epoch takes 6.81s without MonoLoss, 7.08s with our linear-time MonoLoss, and 1127.35s with the quadratic reference implementation, so our method adds only 4\% training-time overhead while achieving a $159\times$ speedup over the reference.

\begin{figure}[ht]
    \centering
    \includegraphics[width=\linewidth]{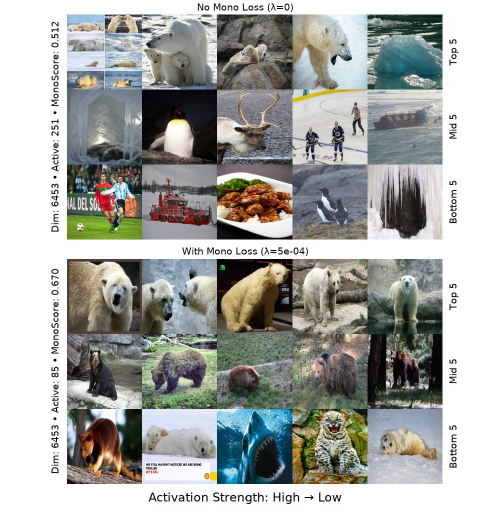}
    \caption{
    Comparison of activation patterns with and without MonoLoss for a BatchTopK autoencoder on CLIP-image features. Both settings show latent 6453, ranked 426 without MonoLoss and 592 with MonoLoss by validation-set monosemanticity. Each row displays the top-5, middle-5, and bottom-5 activated samples of this latent, ordered by activation strength from high to low. Without MonoLoss, the dimension appears monosemantic with respect to polar-bear/ice-related only for the strongest activations, while weaker ones include unrelated concepts like sports and food. With MonoLoss, the same latent mostly focuses on polar bears and bears across top, middle, and bottom activations, indicating better monosemanticity over the full range of activation strengths.
    }
    \label{fig:qualitative_batch_topk_dim6453}
\end{figure}

\section{Experiments}

We assess MonoLoss in two settings. In Section \ref{sec:SAE}, we apply MonoLoss to multiple different SAEs, vision encoders, and datasets. In Section \ref{sec:finetuning}, MonoLoss is adopted as a regularizer to finetune vision encoders for downstream tasks.

\subsection{SAE with MonoLoss}
\label{sec:SAE}

\begin{figure*}[ht]
    \centering
    \includegraphics[width=\linewidth]{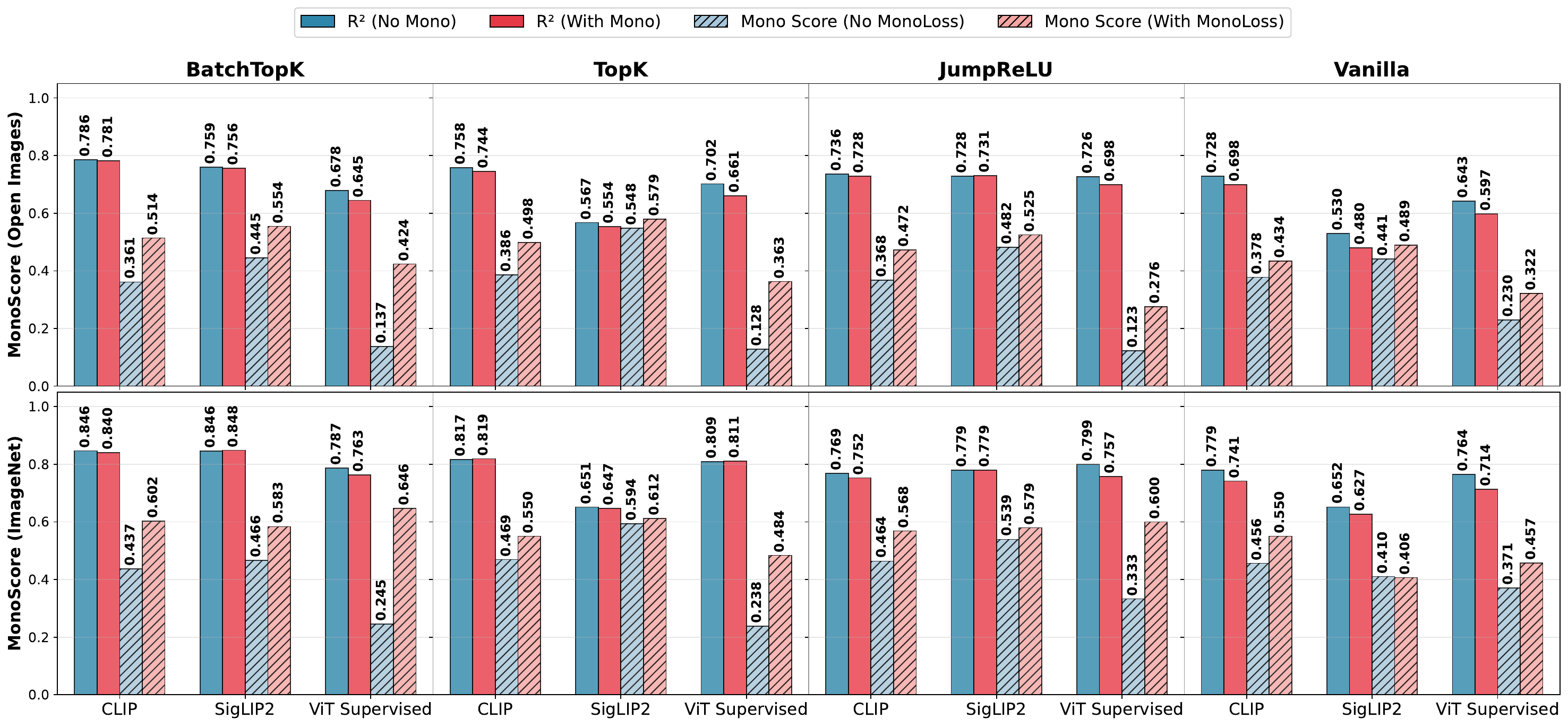}
    \caption{$R^2$ and MonoScore for SAEs trained with and without MonoLoss across four architectures (BatchTopK, TopK, JumpReLU, Vanilla ReLU), three vision encoders (CLIP, SigLIP2, ViT Supervised), and two datasets (Open Images V7 test set, ImageNet-1K validation set). MonoLoss consistently raises monosemanticity across nearly all configurations, with BatchTopK exhibiting minimal reconstruction drops. Supervised ViT shows the most pronounced interpretability--reconstruction trade-off, but also achieves the largest absolute gains in MonoScore. Notably, several encoder--architecture pairs on ImageNet-1K (BatchTopK--SigLIP2, TopK--CLIP, TopK--ViT) show slight $R^2$ improvements with MonoLoss.}

    \label{fig:mono_loss_comparison}
\end{figure*}

We aim to investigate whether explicitly encouraging monosemanticity during training improves the interpretability of SAEs, and to what extent this comes at the cost of reconstruction quality. We therefore compare pairs of models trained with and without MonoLoss, holding all other hyperparameters fixed, and evaluate them on: (i) reconstruction and the monosemanticity score, (ii) class purity (what fraction of images activating a latent belong to its dominant class), and (iii) qualitative inspection of neurons across their full activation range (from strongly to weakly activating samples).

\textbf{Setup and models.}

We train and test sparse autoencoders on OpenImagesV7 \cite{OpenImages} (train/validation/test), and ImageNet-1K \cite{imagenet15russakovsky} (train/validation) using their standard splits. For each image, we extract frozen features from three vision encoders: (i) CLIP ViT-L/14 \cite{clip}, (ii) SigLIP2 \cite{SigLIP2}, and (iii) a supervised ViT-L/16 \cite{vit}. On top of each encoder, we instantiate four SAE architectures with $d = 8192$ latent dimensions: TopK SAE \cite{topksae} with $k = 64$ active latents per example; BatchTopK SAE \cite{batchtopksae} with $k = 64$ active latents per batch; Vanilla ReLU SAE \cite{reluSAE} with $\ell_1 = 0.0001$ sparsity penalty; and JumpReLU SAE \cite{jumprelu} with learnable thresholds. All models are trained with a batch size of 2048, with and without MonoLoss (differing only in the coefficient $\lambda$). Implementation details are provided in the Appendix \ref{app:sae_training_details}.

\textbf{Reconstruction and monosemanticity.}

For each encoder--SAE pair, we train with and without MonoLoss, reporting $R^2$ (the coefficient of determination measuring reconstruction quality) and mean MonoScore on the Open Images test set and ImageNet-1K validation set (Figure~\ref{fig:mono_loss_comparison}). Across nearly all configurations, MonoLoss improves monosemanticity with modest reconstruction cost. BatchTopK achieves the best trade-off: $R^2$ drops by less than one percentage point for CLIP and SigLIP2, while MonoScore increases substantially (e.g., 0.361 to 0.514 for CLIP on Open Images). In several cases---BatchTopK--SigLIP2, TopK--CLIP, and TopK--ViT on ImageNet-1K; JumpReLU--SigLIP2 on Open Images---MonoLoss even slightly improves $R^2$. Supervised ViT exhibits the sharpest trade-off: the largest $R^2$ reductions but also the largest MonoScore gains (0.245 to 0.646 for BatchTopK on ImageNet-1K). Vanilla ReLU incurs the greatest reconstruction penalty and is the only configuration where MonoLoss slightly reduces MonoScore (SigLIP2 on ImageNet-1K).

\textbf{Qualitative Results.} We complement these metrics with visual inspection of individual neurons. For a given SAE architecture on CLIP-image features, we rank latents in each model by their monosemanticity score on the validation set and then either (i) select the same latent index in both runs or (ii) compare latents with the same rank in this ordering. For each chosen latent, we collect all validation examples, sort them by activation strength, and visualize three bands of positively activated images (top, middle, and bottom; from high to low activation). This setup, illustrated in Figures~\ref{fig:qualitative_batch_topk_dim6453} and \ref{fig:qualitative_batch_topk_rank392}, allows us to assess how consistently a neuron responds across its full activation range rather than only at the extreme tail. Qualitatively, SAEs trained with MonoLoss often display more semantically homogeneous activation patterns, particularly for mid-ranked latents where the baseline still activates frequently but tends to mix several visual motifs. For the very top-ranked latents, both models are usually already quite clean, while at the lowest ranks, both may remain diffuse. We provide additional examples in the Appendix \ref{app:sae_qualitative_examples}.

\begin{figure*}[ht]
\centering
\includegraphics[width=\textwidth]{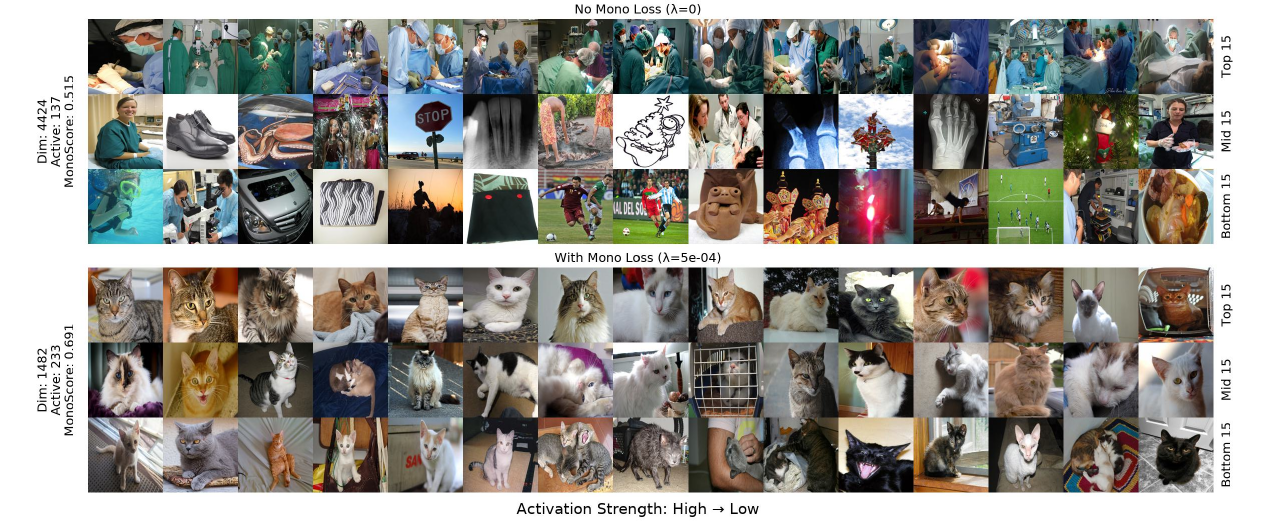}
\caption{
Comparison of activation patterns with and without monosemanticity loss for BatchTopK autoencoders on CLIP-image features. The two rows show latents of the rank ($392$ out of $8192$) in the no-MonoLoss and MonoLoss settings, respectively. For each latent, we show top, middle, and bottom positively activated samples, ordered by activation strength from high to low. The latent from the baseline (no MonoLoss) targets ``surgeons and operating rooms" but rapidly loses coherence, firing on unrelated images like ``shoes, x-rays, and soccer" in weaker activations. In contrast, the latent from the model trained with MonoLoss remains highly coherent, consistently identifying ``cats" across its full dynamic range.}
\label{fig:qualitative_batch_topk_rank392}
\end{figure*}

\textbf{Class-level purity.}
The qualitative comparison in Figure~\ref{fig:qualitative_batch_topk_dim6453} suggests that MonoLoss produces latents with more coherent activation patterns across their dynamic range. To quantify this at scale, we measure class purity on ImageNet-1K validation. For each latent $k$ with at least one positive activation, let $a_{kn}$ denote its activation on sample $n$ with label $y_n$. We assign latent $k$ to its dominant class $c^*_k = \arg\max_c \sum_{n: y_n = c} \mathbf{1}[a_{kn} > 0]$. Let $N_k = \{n : a_{kn} > 0\}$ be the set of samples that activate latent $k$, and $N_k^* = \{n \in N_k : y_n = c^*_k\}$ the subset belonging to the dominant class. We define \emph{binary purity} as $|N_k^*| / |N_k|$ and \emph{weighted purity} as $\left(\sum_{n \in N_k^*} a_{kn}\right) / \left(\sum_{n \in N_k} a_{kn}\right)$.

Table~\ref{tab:class_purity_full} reports mean purity across active latents. MonoLoss improves both metrics in all 12 encoder--SAE combinations, with varying magnitude. Supervised ViT shows the largest gains: BatchTopK improves from $0.152 \to 0.723$ (binary) and $0.331 \to 0.795$ (weighted). SigLIP exhibits the highest baseline purity but the smallest gains. Among architectures, BatchTopK and JumpReLU achieve the highest post-MonoLoss purity, while Vanilla ReLU yields near-zero purity across all encoders (binary $\leq 0.077$) with only marginal improvement from MonoLoss. We note that class purity inherits the label bias of ImageNet-1K; these results demonstrate relative improvement from MonoLoss rather than absolute monosemanticity.
\begin{table*}[ht]
    \centering
    \small
    \setlength{\tabcolsep}{4pt}
    \begin{tabular}{ll|ccc|ccc|ccc|ccc}
        \toprule
        & & \multicolumn{3}{c|}{\textbf{BatchTopK}} & \multicolumn{3}{c|}{\textbf{TopK}} & \multicolumn{3}{c|}{\textbf{JumpReLU}} & \multicolumn{3}{c}{\textbf{Vanilla}} \\
        \cmidrule(lr){3-5} \cmidrule(lr){6-8} \cmidrule(lr){9-11} \cmidrule(lr){12-14}
        \textbf{Enc.} & & Base & Mono & $\Delta$ & Base & Mono & $\Delta$ & Base & Mono & $\Delta$ & Base & Mono & $\Delta$ \\
        \midrule
        \multirow{2}{*}{CLIP} 
            & Bin & 0.144 & \textbf{0.462} & \gd{+0.318} & 0.159 & \textbf{0.251} & \gd{+0.092} & 0.208 & \textbf{0.446} & \gd{+0.238} & 0.002 & \textbf{0.007} & \gd{+0.005} \\
            & Wtd & 0.217 & \textbf{0.539} & \gd{+0.323} & 0.240 & \textbf{0.339} & \gd{+0.099} & 0.278 & \textbf{0.532} & \gd{+0.255} & 0.009 & \textbf{0.024} & \gd{+0.016} \\
        \midrule
        \multirow{2}{*}{SigLIP} 
            & Bin & 0.233 & \textbf{0.326} & \gd{+0.093} & 0.340 & \textbf{0.416} & \gd{+0.076} & 0.249 & \textbf{0.337} & \gd{+0.087} & 0.037 & \textbf{0.077} & \gd{+0.041} \\
            & Wtd & 0.305 & \textbf{0.404} & \gd{+0.100} & 0.400 & \textbf{0.473} & \gd{+0.073} & 0.321 & \textbf{0.409} & \gd{+0.088} & 0.069 & \textbf{0.141} & \gd{+0.071} \\
        \midrule
        \multirow{2}{*}{ViT} 
            & Bin & 0.152 & \textbf{0.723} & \gd{+0.572} & 0.128 & \textbf{0.431} & \gd{+0.303} & 0.287 & \textbf{0.609} & \gd{+0.323} & 0.002 & \textbf{0.004} & \gd{+0.001} \\
            & Wtd & 0.331 & \textbf{0.795} & \gd{+0.463} & 0.319 & \textbf{0.586} & \gd{+0.267} & 0.445 & \textbf{0.767} & \gd{+0.322} & 0.011 & \textbf{0.020} & \gd{+0.009} \\
        \bottomrule
    \end{tabular}
    \caption{Class purity on ImageNet validation for all encoder--SAE combinations, computed over active latents.
    Encoders: CLIP (ViT-L/14), SigLIP (SigLIP2), ViT (ViT-L/16).
    Metrics: Bin (binary purity), Wtd (weighted purity). 
    MonoLoss increases purity in all 12 settings (\gd{+}) for both metrics. $\Delta$ values are computed from unrounded data. Purity metrics inherit ImageNet label bias and reflect relative improvement.}
    \label{tab:class_purity_full}
\end{table*}

\subsection{Finetuning vision models with MonoLoss}
\label{sec:finetuning}

    We assess the impact of MonoLoss as an auxiliary regularizer during fine-tuning by examining both task performance and representation quality. We conduct experiments on image classification with two settings: (i) a baseline with no MonoLoss regularization (only Cross Entropy loss) and (ii) our finetuning setup with MonoLoss as a regularizer.
    
    \textbf{Experimental setup.} We utilized the pretrained ResNet50 \cite{resnet50} and CLIP-ViT-B/32 \cite{clip} and finetuning on CIFAR-10 and CIFAR-100 \cite{cifar10}. During experiments, we found out that directly using MonoLoss to finetune enforces the model to overfit on the fixed subset of images to keep the MonoScore high. It could be because the last layer of the models is not originally trained for monosemantic pattern. To address the problem, we simply add an additional linear layer before the classification head and use the output embeddings from that layer as activations. The objective function for finetuning is $\mathcal{L}$ = \text{CrossEntropy} +\ $\lambda\,\mathcal{L}_{\text{mono}}$ with $\lambda$ is a regularizer coefficient. More details are provided in the Supplementary Material \ref{app:Finetune_settings}.

    \textbf{MonoLoss increases the finetuning performance.} We ablate the regularization weight $\lambda$ for two models (ResNet-50 and CLIP-ViT-B/32) on three benchmark datasets (ImageNet-1K, CIFAR1-0, and CIFAR-100), showing in the Table \ref{tab:finetuning_result}. We find that MonoLoss consistently enhances classification performance. Regarding ImageNet-1K, the approximate gain of 0.6\% and 0.1\% on accuracy are obtained from ResNet-50 ($\lambda{=}0.1$) and CLIP-ViT-B/32 ($\lambda{=}0.1$), respectively. On CIFAR-10, the peak improvements are +0.17\% for ResNet-50 and +0.11\% for CLIP-ViT-B/32, both at $\lambda{=}0.1$. On CIFAR-100, the accuracies are 0.24\% and 0.43\% higher for ResNet-50 ($\lambda{=}0.07$) and CLIP-ViT-B/32 ($\lambda{=}0.05$), respectively.

    \begin{table}[ht]
        \centering
        \small
        \begin{tabular} {c ccc ccc}
            \toprule
            \multirow{2}{*}{\textbf{$\lambda$}} &  \multicolumn{3}{c}{\textbf{ResNet-50}} & \multicolumn{3}{c}{\textbf{CLIP-ViT-B/32}} \\ 
            \cline{2-7}  & 
            IN & CF-10 & CF-100 &
            IN & CF-10 & CF-100 \\         
            \hline
            0.00 & 80.34 & 98.26 & 88.05 & 78.45 & 98.09 & 88.04   \\
            0.03 & 80.29 & 98.37 & 88.18 & 78.50 & 98.08 & 88.32   \\
            0.05 & 80.25 & 98.40 & 88.20 & 78.34 & 98.07 & \textbf{88.47}    \\
            0.07 & 80.31 & 98.38 & \textbf{88.29} & 78.50 & 98.00  & 88.32   \\
            0.10 & \textbf{80.93} & \textbf{98.43} & 88.12 & \textbf{78.52} & \textbf{98.20} & 88.37   \\
            \bottomrule
        \end{tabular}
        \caption{The accuracy in fine-tuning ResNet-50 and CLIP-ViT-B/32 on ImageNet-1K (IN), CIFAR-10 (CF-10), and CIFAR-100 (CF-100). Adding MonoLoss yields consistent performance gains on both models and all datasets.}
        \label{tab:finetuning_result}
    \end{table}
    \textbf{MonoLoss enhances the monosemanticity}. We visualize top-activating images to assess MonoLoss’s effect on concept purity. Specifically, for each neuron, we consider only positive activation values and select top-15, middle-15, and bottom-15. Figure \ref{fig:cifar10_rn50_10} highlights the clear difference of the concept heterogeneity between the no-MonoLoss  (top three rows) and MonoLoss setting (bottom three rows) at rank 10. With MonoLoss, the activations exhibit strong semantic coherence, maintaining class-consistent content not only in the top-15 but even across the bottom-15 samples. It is also noted that the number of activating images is 2837 on the baseline setting, whereas only 73 images having positive activation when finetuning with MonoLoss. This results demonstrate that MonoLoss enforces the neuron to activate only for the set of similar images. Additional examples are provided in the Supplementary Material \ref{app:finetuning_qualitative}.

    \textbf{MonoLoss enhances the performance with 1 minute overhead.} Table \ref{tab:training_time} compares per-epoch fine-tuning time across two backbones (ResNet-50, CLIP-ViT-B/32) on two datasets (ImageNet-1K, CIFAR-10) with three settings: no MonoLoss (Cross Entropy only), Cross Entropy with baseline MonoScore, and Cross Entropy with our MonoScore. Our linear-time MonoLoss adds negligible cost relative to the traditional finetuning: on ImageNet-1K it rises +0.42\% and +0.78\% for ResNet-50 and CLIP-ViT-B/32, respectively. On CIFAR-10, the absolute overhead remains tiny: 14\% on ResNet-50 and 0.59\%  on CLIP-ViT-B/32. In contrast, the prior MonoScore formulation imposes substantial overhead: +26.7\% for ResNet-50 and +29.8\% for CLIP-ViT-B/32 on ImageNet-1K, and on CIFAR-10, it is $8.3\times$ slower for ResNet-50 and +32.5\% for CLIP-ViT-B/32. Overall, our formulation preserves the accuracy-monosemanticity gains while keeping training time essentially unchanged.
    
    \begin{table}[h]
        \centering
        \small
        \begin{tabular} {c cc cc}
            \toprule
            \multirow{2}{*}{\textbf{Setting}} &  \multicolumn{2}{c}{\textbf{ResNet-50}} & \multicolumn{2}{c}{\textbf{CLIP-ViT-B/32}} \\ 
            \cline{2-5}  & 
            IN & CF-10 & IN & CF-10 \\         
            \hline
            No MonoLoss & 44.95 & 0.07 & 42.53 & 1.69\\
            MonoLoss (baseline) & 56.94 & 0.58 & 55.21 & 2.24\\
            MonoLoss (ours) & 45.14 & 0.08 & 42.86 & 1.70\\
            \bottomrule
        \end{tabular}
        \caption{Finetuning time per epoch (in minutes) for ResNet-50 and CLIP-ViT-B/32 on ImageNet-1K (IN-1K) and CIFAR-10 (CF-10) with three settings: no MonoLoss, baseline original formulation MonoLoss \cite{monoscore}, and our MonoLoss formulation. The time of CIFAR-100 is approximately equal to CIFAR-10 due to the similar number of samples. Our formulation increases accuracy using only  1\% additional  time, in contrast to the substantial overhead of  the previous formulation (up to 
        760\%).}
        \label{tab:training_time}
    \end{table}

\section*{Acknowledgements}
    This work was partially supported by USA NSF grants IIS-2123920 [D.S], IIS-2212046 [D.S], NSERC-DG RGPIN‐2022‐05378 [M.S.H] and Amazon Research Award [M.S.H].

\section{Conclusions}
    We transformed MonoScore from a costly post-hoc analysis into a practical training signal. By reformulating MonoScore into a single-pass with linear-time computation, we make large-scale evaluation and in-training use feasible. We also introduced MonoLoss, a simple plug-and-play objective that directly rewards semantically consistent activations and integrates seamlessly into standard training. Across SAEs trained on diverse feature sources and architectures, MonoLoss consistently increases monosemanticity and yields higher performance with near-zero computational cost. However, MonoScore’s dependence on pre-extracted embedding similarities makes it sensitive to model quality. In the future, we could study the ensemble of foundation models to improve reliability and ensure gains on MonoScore faithfully reflect semantic consistency.

\section*{Impact Statement}
    This paper presents work whose goal is to advance the field of Machine Learning. There are many potential societal consequences of our work, none which we feel must be specifically highlighted here.

\nocite{langley00}

\bibliography{main}
\bibliographystyle{icml2026}

\newpage
\appendix
\onecolumn



\section{Proof of Proposition~\ref{prop:monoscore-equivalence}}
\label{app:proof-monoscore-equivalence}

\begin{proof}[Proof of Proposition~\ref{prop:monoscore-equivalence}]
Consider a single neuron $k$. Let $\tilde a_{ki}$ be the min--max normalized activation
of neuron $k$ on image $x_i$, and let $h_i \in \mathbb{R}^d$ be the
unit-normalized embedding $h_i = E(x_i) / \|E(x_i)\|_2$.

The pairwise MonoScore (Algorithm~\ref{alg:monoscore-baseline}) is
\[
    MS^{\text{pair}}_k
    =
    \frac{\displaystyle\sum_{1 \le i < j \le N}
        \tilde a_{ki}\,\tilde a_{kj}\, h_i^\top h_j}
         {\displaystyle\sum_{1 \le i < j \le N}
        \tilde a_{ki}\,\tilde a_{kj}}.
\]

The linear-time algorithm (Algorithm~\ref{alg:monoscore-linear}) maintains
\[
    u_k = \sum_{i=1}^N \tilde a_{ki}, \qquad
    v_k = \sum_{i=1}^N \tilde a_{ki}^2, \qquad
    w_{:k} = \sum_{i=1}^N \tilde a_{ki} h_i ,
\]
and defines
\[
    q_k = \|w_{:k}\|_2^2 .
\]

Expanding $q_k$ gives:
\[
    q_k = \|w_{:k}\|_2^2
    = \Big\|\sum_{i=1}^N \tilde a_{ki} h_i \Big\|_2^2
    = \sum_{i=1}^N \sum_{j=1}^N
      \tilde a_{ki}\,\tilde a_{kj}\, h_i^\top h_j.
\]
Splitting into diagonal and off-diagonal terms (given the symmetry of the matrix),
\[
    q_k
    = \sum_{i=1}^N \tilde a_{ki}^2
      + 2 \sum_{1 \le i < j \le N}
        \tilde a_{ki}\,\tilde a_{kj}\, h_i^\top h_j,
\]
so
\[
    \frac{1}{2}\big(q_k - v_k\big)
    = \sum_{1 \le i < j \le N}
      \tilde a_{ki}\,\tilde a_{kj}\, h_i^\top h_j.
\]

Similarly,
\begin{align*}
u_k^2
  &= \Big(\sum_{i=1}^N \tilde a_{ki}\Big)^2 \\
  &= \sum_{i=1}^N \sum_{j=1}^N \tilde a_{ki}\,\tilde a_{kj} \\
  &= \sum_{i=1}^N \tilde a_{ki}^2
   + 2 \sum_{1 \le i < j \le N} \tilde a_{ki}\,\tilde a_{kj}.
\end{align*}
which implies
\[
    \frac{1}{2}\big(u_k^2 - v_k\big)
    = \sum_{1 \le i < j \le N} \tilde a_{ki}\,\tilde a_{kj}.
\]

By definition, Algorithm~\ref{alg:monoscore-linear} sets
\[
    \text{num}_k = \tfrac{1}{2}(q_k - v_k), \qquad
    \text{den}_k = \tfrac{1}{2}(u_k^2 - v_k),
\]
and returns $MS^{\text{lin}}_k = \text{num}_k / \text{den}_k$ when
$\text{den}_k > 0$, and $0$ otherwise. The identities above show that
$\text{num}_k$ and $\text{den}_k$ are exactly the numerator and denominator
of $MS^{\text{pair}}_k$. Thus $MS^{\text{lin}}_k = MS^{\text{pair}}_k$
for all $k$.
\end{proof}

\section{SAE training details}
\label{app:sae_training_details}

We build on publicly available PyTorch implementations of the SAEs we study, and add a common data pipeline, MonoLoss, and unified training and evaluation code. For TopK SAE and BatchTopK SAE, we use the reference implementations released by the original authors \cite{topksae, batchtopksae}. For the remaining SAE variants (Vanilla ReLU SAE and JumpReLU SAE), we use implementations from \cite{batchtopksae}. In all cases, we keep the architectural definitions from these repositories unchanged, and standardize only the optimization hyperparameters and the way MonoLoss is integrated.

We train all models on pre-extracted feature embeddings from frozen vision encoders, following \cite{nasiri2025sparc}. Specifically, we use CLIP ViT-L/14 \cite{clip}, SigLIP2 \cite{SigLIP2}, and a supervised ViT-L/16 \cite{vit}. We train on OpenImagesV7 \cite{OpenImages} using its standard train/validation/test splits, and on ImageNet-1K \cite{imagenet15russakovsky} using the training split (1.28M images) for training and the validation split (50K images) for evaluation. For each dataset split, we pre-compute image features once, store them in LMDB, and normalize every feature vector to unit $\ell_2$ norm prior to training. During training, features are loaded via a dataset wrapper that streams examples from disk and re-normalizes them to unit norm on the fly.

\paragraph{Architectures.}
On top of each frozen encoder, we train four SAE architectures with a dictionary size of $d=8192$ latents:
\begin{itemize}
    \item \textbf{TopK SAE} \cite{topksae}: a linear encoder--decoder with a TopK activation that keeps exactly $k = 64$ non-zero latents per example. We use tied weights by default (decoder is the transpose of the encoder) and optionally apply layer normalization on the input and output. For the MonoLoss runs in Figure~\ref{fig:mono_loss_comparison}, we set the MonoLoss coefficient on OpenImagesV7 to $\lambda = 0.5$ for CLIP-image features, $\lambda = 0.3$ for SigLIP2-image features, and $\lambda = 0.9$ for the supervised ViT encoder. On ImageNet-1K, we use $\lambda = 0.09$ (CLIP), $\lambda = 0.14$ (SigLIP2), and $\lambda = 0.09$ (ViT).
    
    \item \textbf{BatchTopK SAE} \cite{batchtopksae}: a variant where sparsity is enforced at the batch level rather than per example. We compute ReLU activations and then retain the top $k = 64$ activations \emph{per batch} (globally across all samples). For the MonoLoss configurations on OpenImagesV7, we use $\lambda = 5\times 10^{-4}$ for CLIP-image features, $\lambda = 2\times 10^{-4}$ for SigLIP2-image features, and $\lambda = 7\times 10^{-4}$ for the supervised ViT encoder. On ImageNet-1K, we use $\lambda = 3\times 10^{-4}$ (CLIP), $\lambda = 10^{-4}$ (SigLIP2), and $\lambda = 4\times 10^{-4}$ (ViT).

    \item \textbf{Vanilla ReLU SAE}: a standard sparse autoencoder with ReLU activations and an $\ell_1$ sparsity penalty on the latent codes. We fix the $\ell_1$ coefficient to $10^{-4}$ in all experiments. For MonoLoss runs on OpenImagesV7, we use $\lambda = 10^{-4}$ on CLIP-image features, $\lambda = 3\times 10^{-5}$ on SigLIP2-image features, and $\lambda = 10^{-4}$ on the supervised ViT encoder. On ImageNet-1K, we use $\lambda = 10^{-4}$ (CLIP), $\lambda = 3\times 10^{-5}$ (SigLIP2), and $\lambda = 9\times 10^{-5}$ (ViT).

    \item \textbf{JumpReLU SAE} \cite{jumprelu}: we use the JumpReLU variant from \cite{batchtopksae}, with learnable per-feature log-thresholds $\ell_k$ defining positive thresholds $\tau_k = \exp(\ell_k)$ and a JumpReLU activation applied to ReLU pre-activations. In all SAE experiments, we use $d = 8192$ latents, a JumpReLU bandwidth of $b = 10^{-3}$, and a sparsity loss coefficient of $10^{-3}$ on the JumpReLU activations. For the MonoLoss experiments on OpenImagesV7, we set $\lambda = 5\times 10^{-4}$ for CLIP-image features, $\lambda = 10^{-4}$ for SigLIP2-image features, and $\lambda = 5\times 10^{-4}$ for the supervised ViT encoder. On ImageNet-1K, we use $\lambda = 5\times 10^{-4}$ (CLIP), $\lambda = 10^{-4}$ (SigLIP2), and $\lambda = 7\times 10^{-4}$ (ViT).
    
\end{itemize}
All SAEs share the same linear encoder and decoder parameterization (plus biases), and all operate directly in the frozen-encoder feature space. The decoder columns are constrained to have unit $\ell_2$ norm by explicit renormalization after each gradient step; for untied decoders, we also project the gradient to be orthogonal to the current dictionary vectors to avoid shrinking their norms.

\paragraph{Optimization.}
We train each SAE for 50 epochs with a batch size of 2048 and dictionary size $d = 8192$. We use the Adam optimizer \cite{kingma2015adam} with a learning rate of $1\times 10^{-4}$ and $\epsilon = 6.25\times 10^{-10}$; all other Adam hyperparameters are left at their PyTorch defaults. We apply gradient clipping with a maximum global norm of $1.0$ before the optimizer step, and we do not use any learning-rate schedule or weight decay. For the BatchTopK, Vanilla, and JumpReLU SAEs, we follow the public \cite{batchtopksae} implementation for weight initialization and other training details, and for the TopK SAE, we follow the reference implementation from \cite{topksae}. For all experiments, we fix the random seed to 42 and train on a single GPU. 

For TopK SAEs, we additionally track, for each latent dimension, the number of optimization steps since it last fired on any sample in the current batch. When this counter exceeds a preset threshold (specified in training examples), we reinitialize that latent using the dead-neuron handling strategy from \cite{topksae}, and reset its inactivity counter. This prevents permanently dead features and keeps the effective dictionary size closer to $d$.

\paragraph{Evaluation metrics.}
Reconstruction quality is measured using the coefficient of determination $R^2$ on held-out data. Given ground-truth features $x$ and reconstructions $\hat x$, we compute per-dimension $R^2$ using streaming estimates of means and variances and report the mean $R^2$ across all feature dimensions. For monosemanticity, we report the mean MonoScore over all latents on the test split (OpenImagesV7) or validation split (ImageNet-1K), together with MonoScore curves where latents are sorted by decreasing MonoScore and the index is normalized to $[0,1]$. All reconstruction and monosemanticity results in Section~\ref{sec:SAE} are obtained from the same checkpoints and training procedure described above, differing only in the choice of architecture, dataset, and MonoLoss coefficient $\lambda$.

\section{Qualitative samples for SAE}
\label{app:sae_qualitative_examples}
In this section, we show additional neuron visualizations for BatchTopK and TopK SAEs, comparing models trained with and without MonoLoss. All the results in this section are for the models trained on OpenImages. For selected latents, we rank validation images by activation and display top, middle, and bottom positively activated samples. Across figures \ref{fig:qualitative_batch_topk_rank9}–\ref{fig:qualitative_batch_topk_rank6015} and \ref{fig:qualitative_topk_rank66}–\ref{fig:qualitative_topk_rank6678}, baseline SAEs often look clean only at the very top activations and become polysemantic lower in the range, mixing unrelated concepts. The corresponding MonoLoss-trained latents typically maintain a single, coherent concept (e.g.\ cats, jellyfish, ships) over a much larger portion of their activation spectrum, visually confirming the quantitative MonoScore and Class Purity improvements reported in the main text.
\begin{figure*}[ht]
\centering
\includegraphics[width=\textwidth]{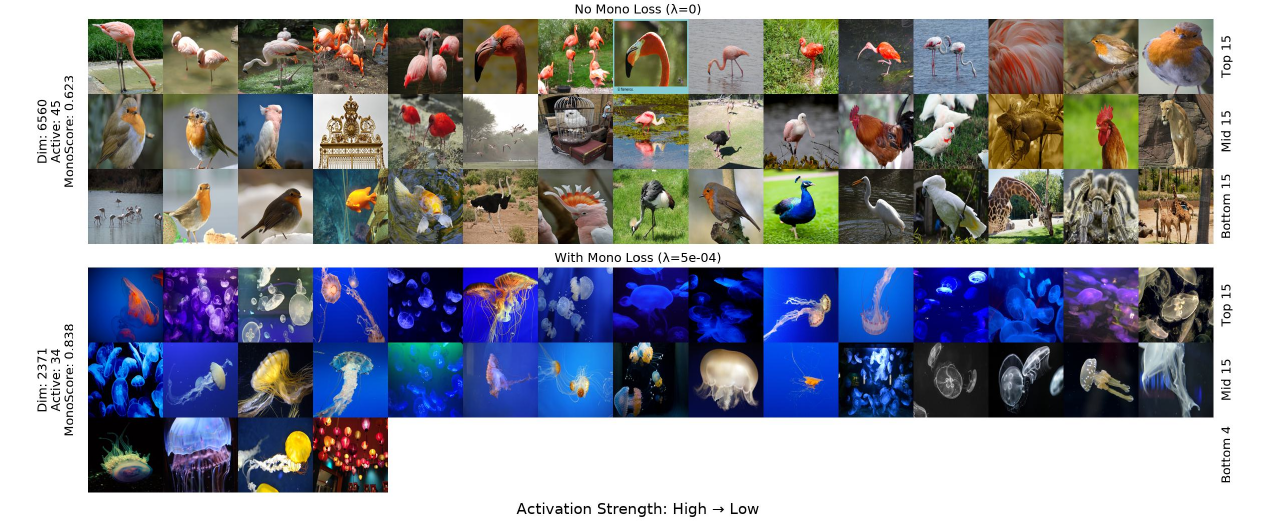}
\caption{
Comparison of activation patterns for BatchTopK SAEs at rank 9. The latent from the baseline (no MonoLoss) triggers strongly on ``flamingos" but exhibits polysemantic behavior in weaker activations, mixing in unrelated birds and scenic backgrounds. In contrast, the latent from the model trained with MonoLoss  maintains high semantic purity, activating exclusively on ``jellyfish" across its full dynamic range.
}
\label{fig:qualitative_batch_topk_rank9}
\end{figure*}

\begin{figure*}[ht]
\centering
\includegraphics[width=\textwidth]{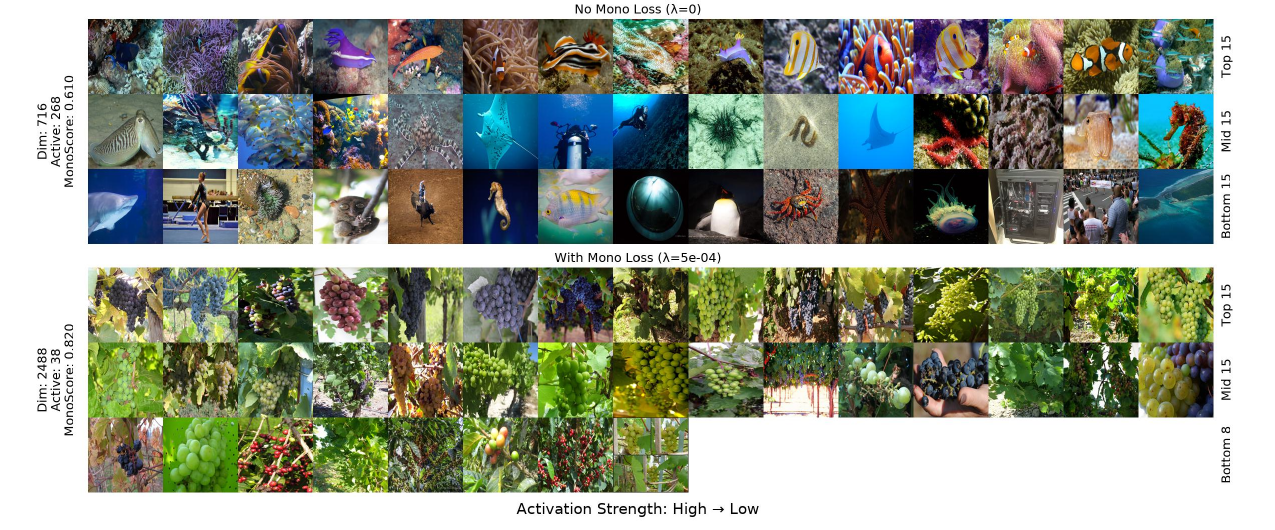}
\caption{
Comparison of activation patterns for BatchTopK SAEs at rank 16. The latent from the baseline (no MonoLoss) targets ``marine life and coral reefs" but becomes less monosemantic in weaker activations, firing on unrelated concepts such as ``gymnasts and computer hardware". In contrast, the latent from the model trained with MonoLoss demonstrates robust coherence, consistently identifying ``grapes and vineyards" across its full dynamic range.
}
\label{fig:qualitative_batch_topk_rank16}
\end{figure*}

\begin{figure*}[ht]
\centering
\includegraphics[width=\textwidth]{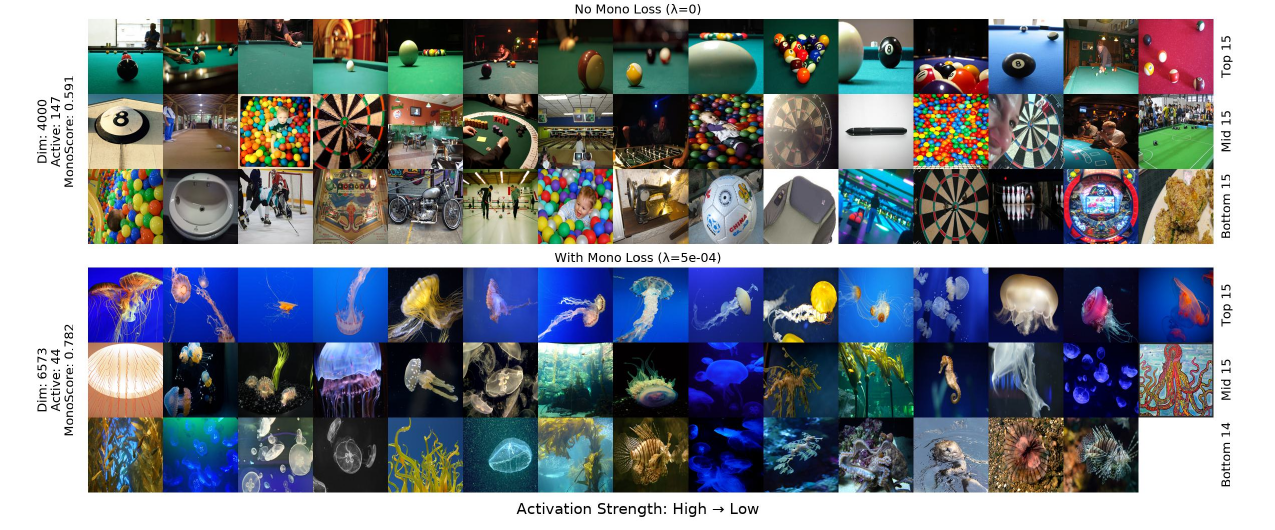}
\caption{
Comparison of activation patterns for BatchTopK SAEs at rank 46. The latent from the baseline (no MonoLoss) targets ``billiards" but becomes less monosemantic in weaker activations, firing on unrelated concepts such as ``dartboards and ball pits". In contrast, the latent from the model trained with MonoLoss demonstrates robust coherence, consistently identifying ``jellyfish and marine life" across its full dynamic range.
}
\label{fig:qualitative_batch_topk_rank46}
\end{figure*}

\begin{figure*}[ht]
\centering
\includegraphics[width=\textwidth]{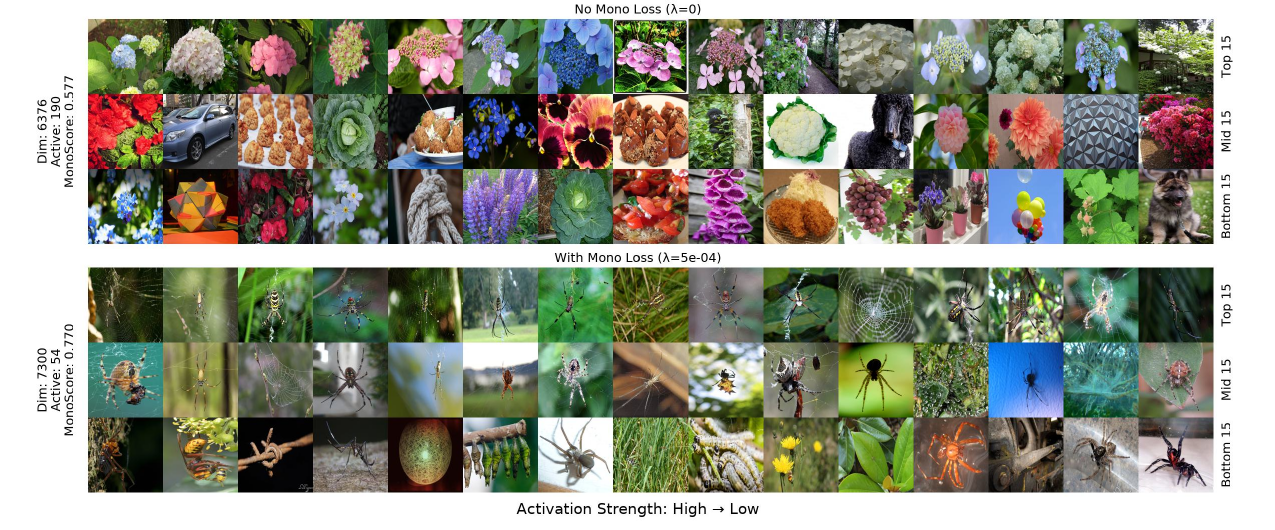}
\caption{
Comparison of activation patterns for BatchTopK SAEs at rank 68. The latent from the baseline (no MonoLoss) targets ``hydrangeas" but exhibits severe polysemantic drift in weaker activations, mixing in ``cars, food, and animals." In contrast, the latent from the model trained with MonoLoss demonstrates higher coherence, primarily targeting ``spiders" across the activation spectrum.
}
\label{fig:qualitative_batch_topk_rank68}
\end{figure*}

\begin{figure*}[ht]
\centering
\includegraphics[width=\textwidth]{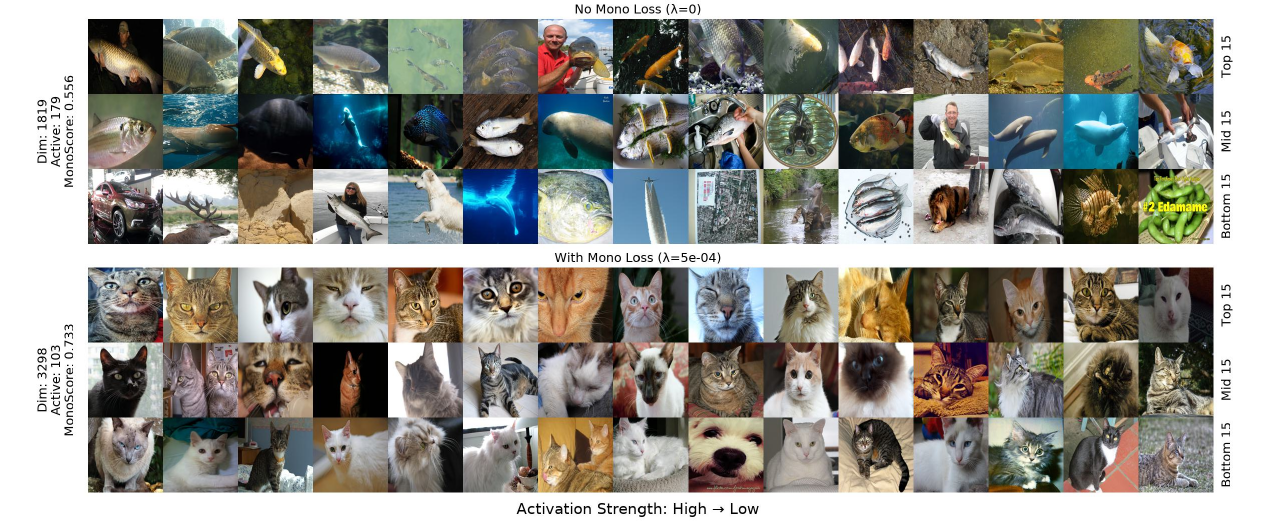}
\caption{
Comparison of activation patterns for BatchTopK SAEs at rank 145. The latent from the baseline (no MonoLoss) targets ``fish" but displays substantial polysemantic drift in weaker activations, capturing unrelated concepts like ``cars, deer, and vegetables." In contrast, the latent from the model trained with MonoLoss remains coherent, activating on ``cats" across its full dynamic range.
}
\label{fig:qualitative_batch_topk_rank145}
\end{figure*}

\begin{figure*}[ht]
    \centering
    \includegraphics[width=\textwidth]{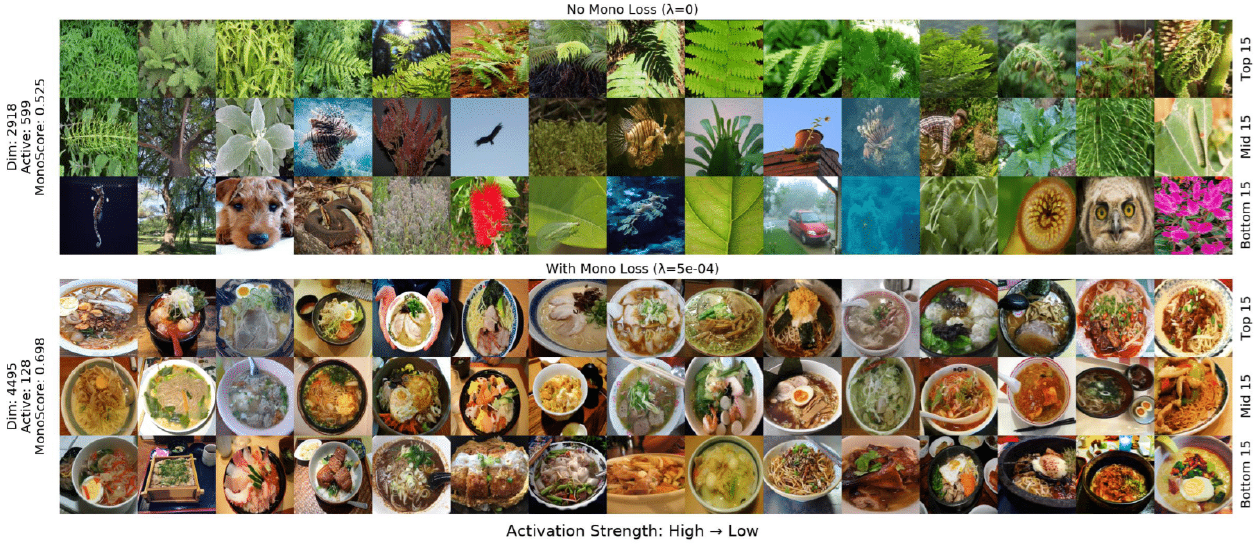}
    \caption{
    Comparison of activation patterns for BatchTopK SAEs at rank 327. The latent from the baseline (no MonoLoss) targets “ferns and green foliage” but loses semantic specificity in weaker activations, mixing in unrelated images like “dogs, owls, and vehicles.” In contrast, the latent from the model trained with MonoLoss focuses on “ramen and noodle soups,” while maintaining high semantic consistency with “similar Asian cuisine and food bowls” even in the lower range.}

\end{figure*}

\begin{figure*}[ht]
\centering
\includegraphics[width=\textwidth]{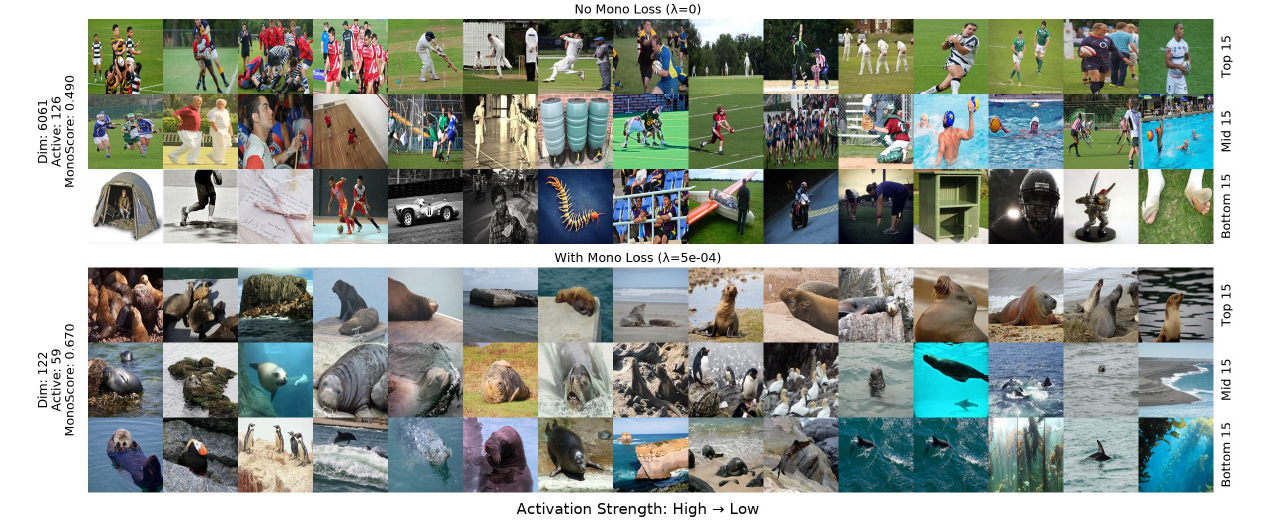}
\caption{
Comparison of activation patterns for BatchTopK SAEs at rank 600. The latent from the baseline (no MonoLoss) targets ``field sports" like rugby but loses semantic specificity in weaker activations, mixing in unrelated images like ``centipedes and helmets." In contrast, the latent from the model trained with MonoLoss focuses on ``seals and sea lions," while also capturing related concepts like ``penguins and coastal environments" in the lower range.
}
\label{fig:qualitative_batch_topk_rank600}
\end{figure*}

\begin{figure*}[ht]
\centering
\includegraphics[width=\textwidth]{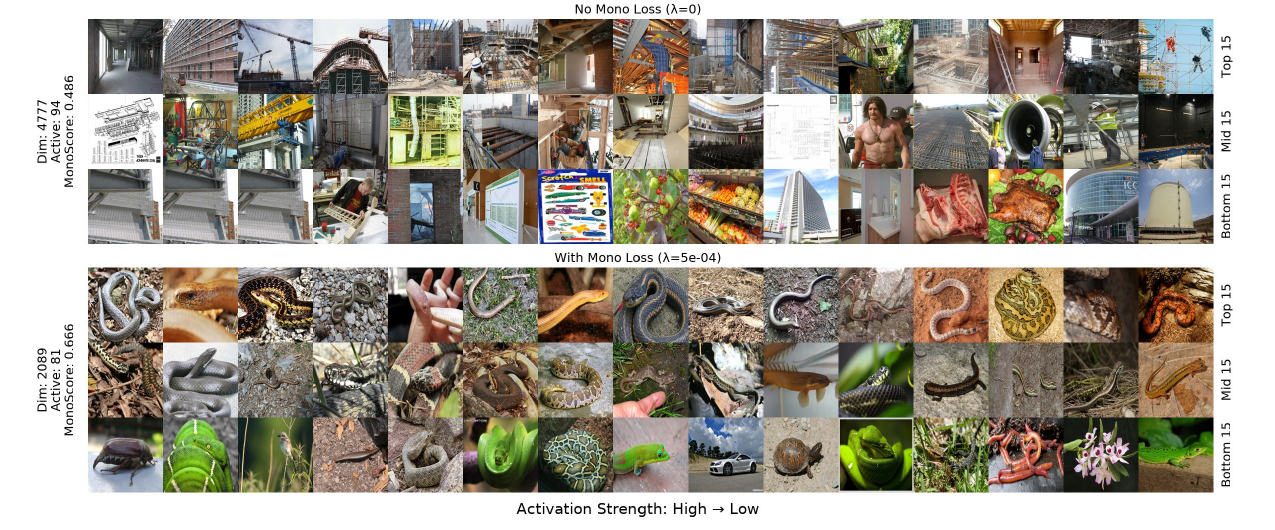}
\caption{
Comparison of activation patterns for BatchTopK SAEs at rank 636. The latent from the baseline (no MonoLoss) targets ``construction sites" but loses semantic coherence in weaker activations, drifting into unrelated concepts like ``food and meat." In contrast, the latent from the model trained with MonoLoss primarily targets ``snakes and reptiles," retaining the core concept across the majority of its dynamic range.
}
\label{fig:qualitative_batch_topk_rank636}
\end{figure*}

\begin{figure*}[ht]
\centering
\includegraphics[width=\textwidth]{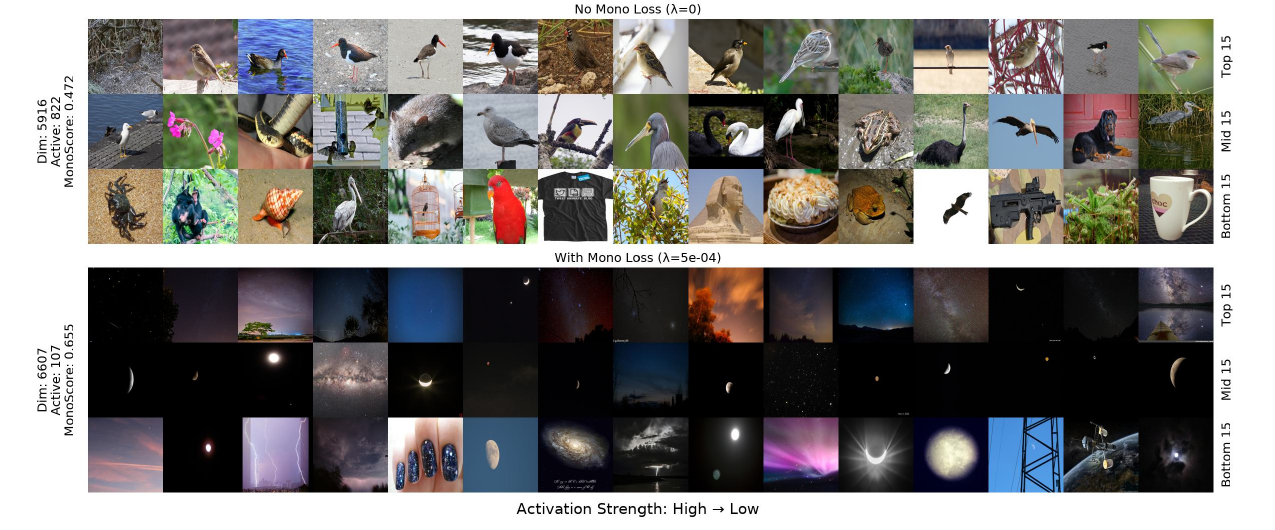}
\caption{
Comparison of activation patterns for BatchTopK SAEs at rank 789. The latent from the baseline (no MonoLoss) targets ``birds" but loses specificity in weaker activations, mixing in diverse animals and unrelated objects like ``guns and food." In contrast, the latent from the model trained with MonoLoss primarily targets ``night sky and celestial bodies," retaining the thematic concept across the majority of its dynamic range.
}
\label{fig:qualitative_batch_topk_rank789}
\end{figure*}

\begin{figure*}[ht]
\centering
\includegraphics[width=\textwidth]{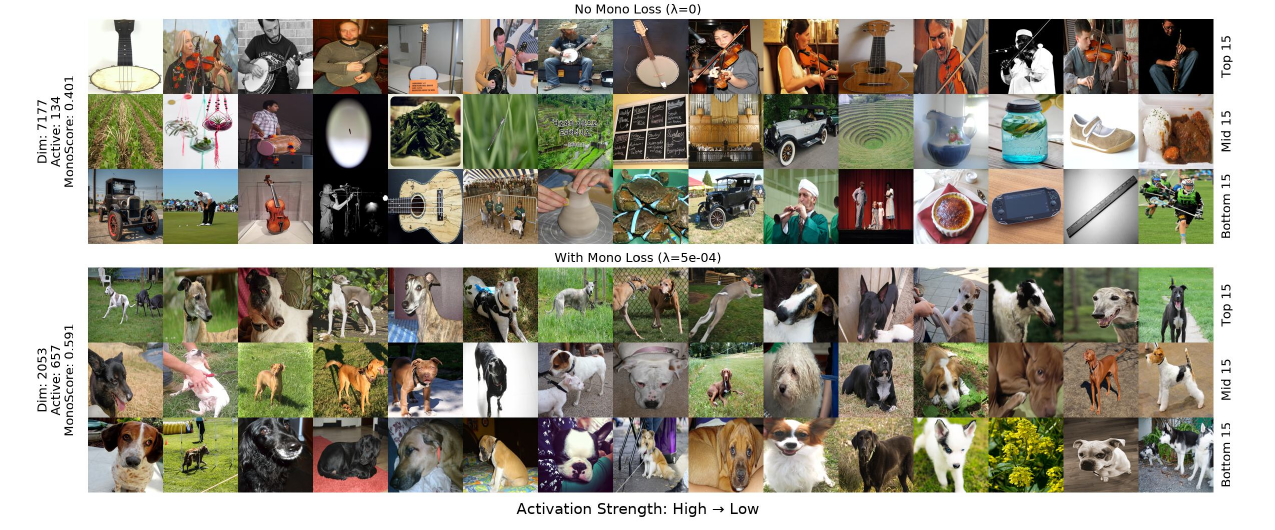}
\caption{
Comparison of activation patterns for BatchTopK SAEs at rank 2175. The latent from the baseline (no MonoLoss) targets ``string instruments" but exhibits polysemantic drift in weaker activations, mixing in unrelated objects like ``cars and food." In contrast, the latent from the model trained with MonoLoss demonstrates robust coherence, consistently identifying ``dogs" across the majority of its dynamic range.
}
\label{fig:qualitative_batch_topk_rank2175}
\end{figure*}

\begin{figure*}[ht]
\centering
\includegraphics[width=\textwidth]{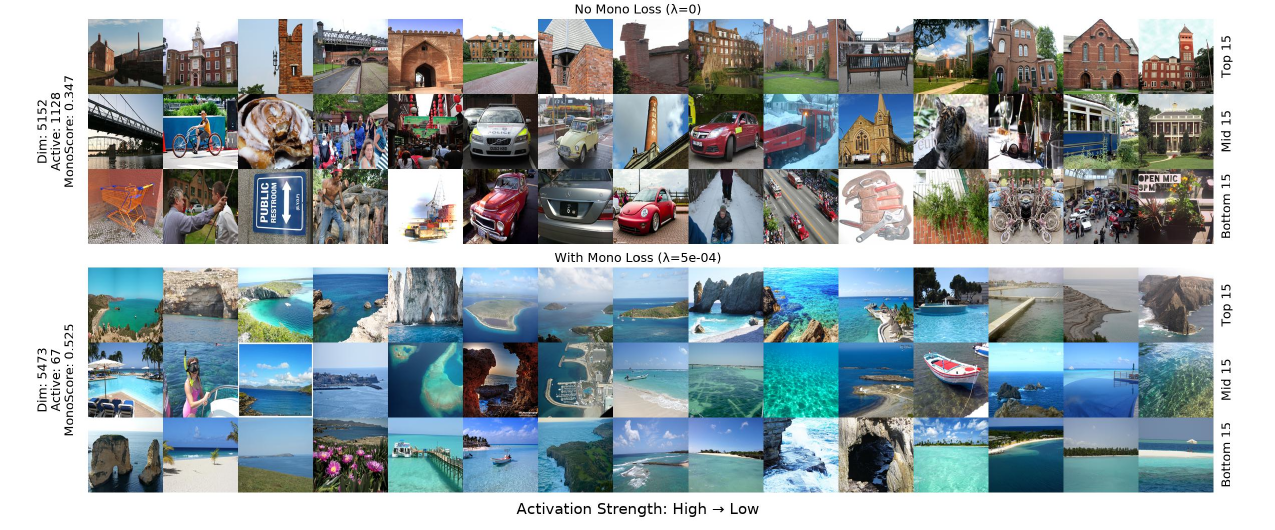}
\caption{
Comparison of activation patterns for BatchTopK SAEs at rank 4317. The latent from the baseline (no MonoLoss) targets ``brick buildings" but exhibits polysemantic drift in weaker activations, mixing in unrelated concepts like ``cars and animals." In contrast, the latent from the model trained with MonoLoss demonstrates robust coherence, consistently identifying ``coastal landscapes and ocean scenes" across its full dynamic range.
}
\label{fig:qualitative_batch_topk_rank4317}
\end{figure*}

\begin{figure*}[ht]
\centering
\includegraphics[width=\textwidth]{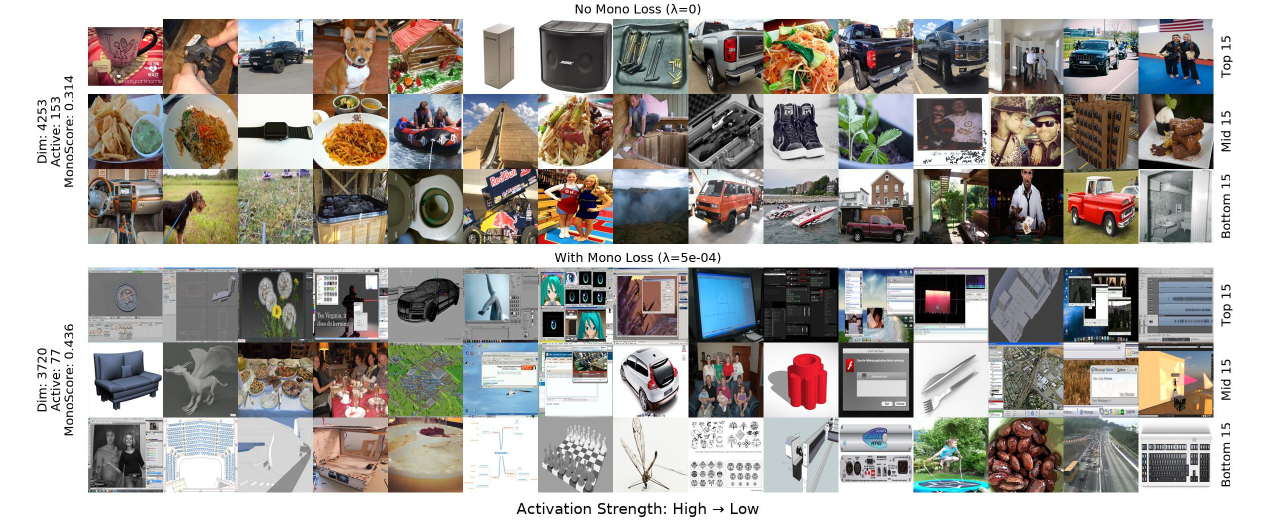}
\caption{
Comparison of activation patterns for BatchTopK SAEs at rank 6015 (lower-quality tail). The latent from the baseline (no MonoLoss) is highly polysemantic, triggering on an incoherent mix of ``trucks, food, and household objects." In contrast, even at this lower rank, the latent from the model trained with MonoLoss retains a discernible theme, primarily targeting ``computer interfaces and 3D modeling software," albeit with some noise in the weaker activations.
}
\label{fig:qualitative_batch_topk_rank6015}
\end{figure*}


\begin{figure*}[ht]
\centering
\includegraphics[width=\textwidth]{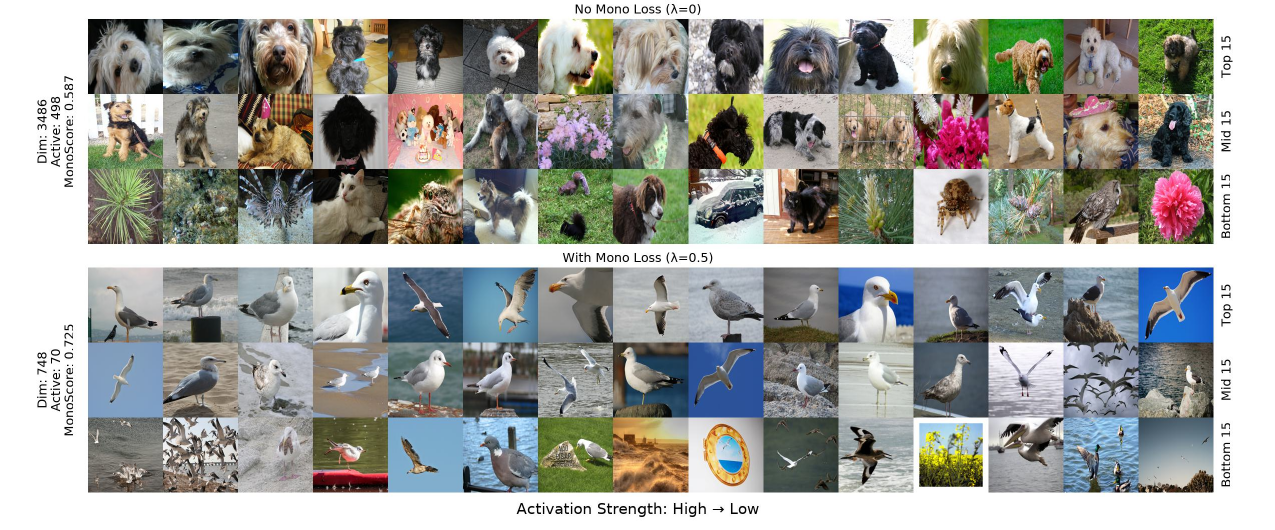}
\caption{
Comparison of activation patterns for TopK SAEs at rank 66. The latent from the baseline (no MonoLoss) targets ``dogs" but becomes polysemantic in weaker activations, mixing in unrelated concepts like ``pine needles, spiders, and cars." In contrast, the latent from the model trained with MonoLoss demonstrates robust coherence, consistently identifying ``seagulls" across the majority of its dynamic range.}
\label{fig:qualitative_topk_rank66}
\end{figure*}

\begin{figure*}[ht]
\centering
\includegraphics[width=\textwidth]{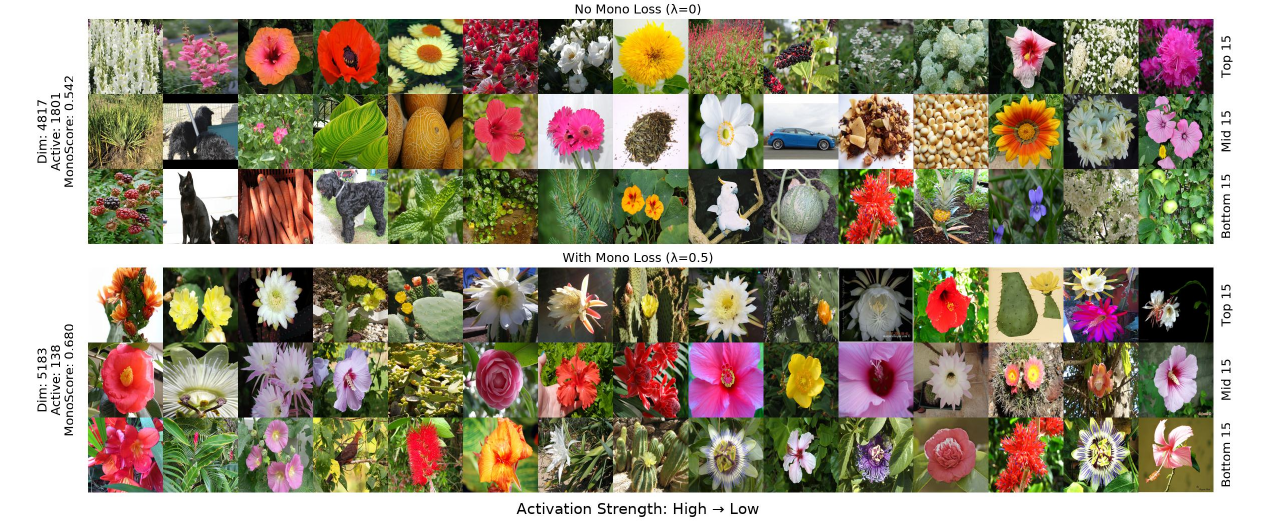}
\caption{
Comparison of activation patterns for TopK SAEs at rank 243. The latent from the baseline (no MonoLoss) targets ``flowers" but exhibits polysemantic drift in weaker activations, mixing in unrelated concepts like ``black dogs, cars, and produce." In contrast, the latent from the model trained with MonoLoss demonstrates robust coherence, consistently identifying ``flowers and blooms" across its full dynamic range.
}
\label{fig:qualitative_topk_rank243}
\end{figure*}

\begin{figure*}[ht]
\centering
\includegraphics[width=\textwidth]{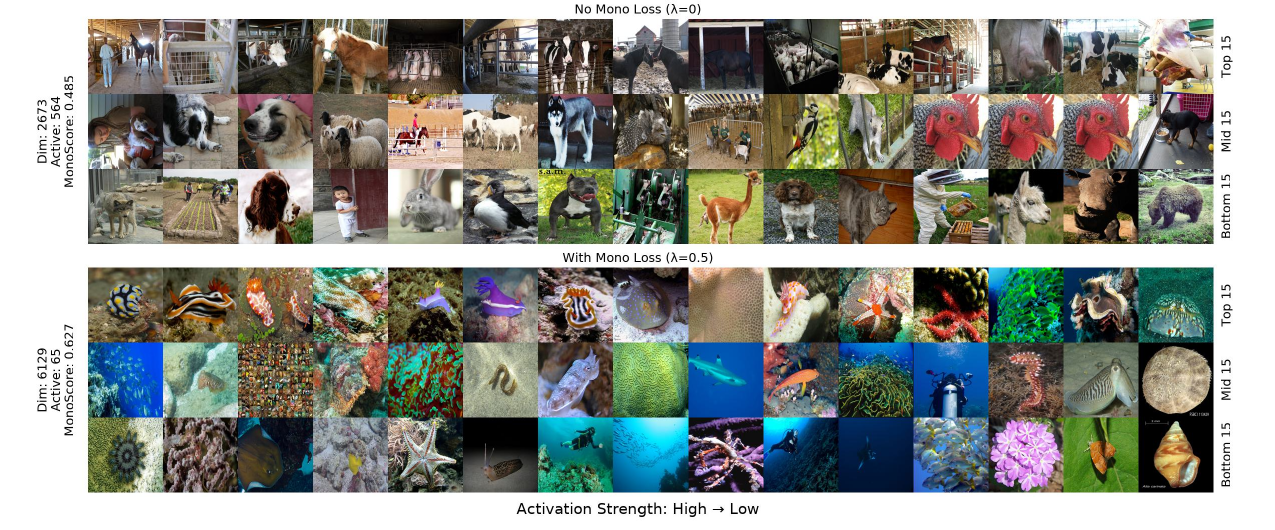}
\caption{
Comparison of activation patterns for TopK SAEs at rank 756. The latent from the baseline (no MonoLoss) targets ``livestock" but loses semantic specificity in weaker activations, mixing in diverse animals like ``dogs, bears, and birds." In contrast, the latent from the model trained with MonoLoss consistently identifies ``marine life and coral reefs," maintaining a coherent underwater theme across its full dynamic range.
}
\label{fig:qualitative_topk_rank756}
\end{figure*}

\begin{figure*}[ht]
\centering
\includegraphics[width=\textwidth]{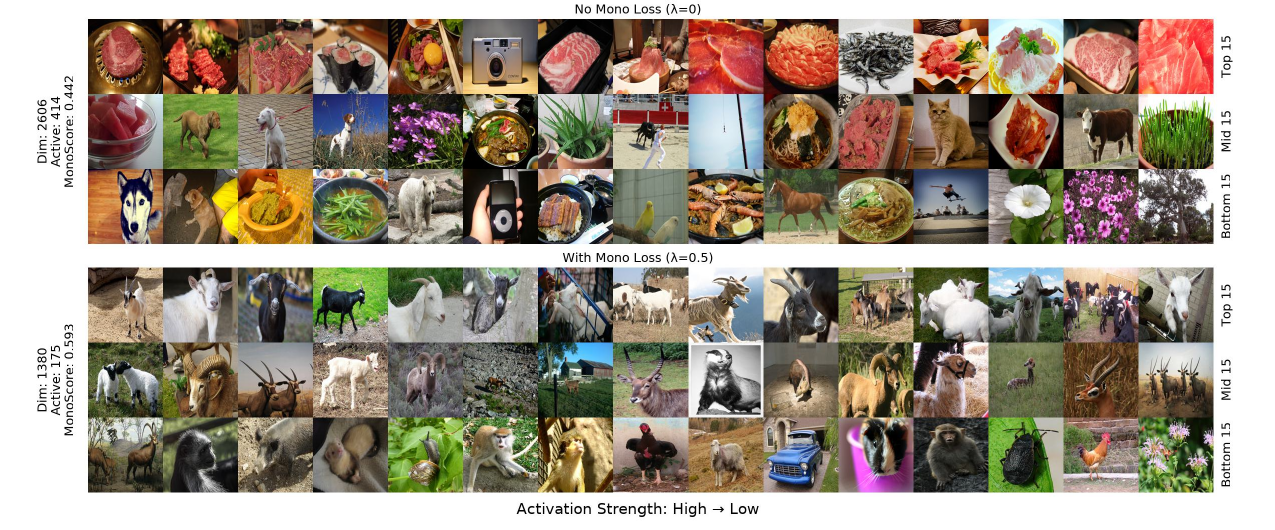}
\caption{
Comparison of activation patterns for TopK SAEs at rank 1428. The latent from the baseline (no MonoLoss) targets ``meat and food" but rapidly degrades into an incoherent mix of ``dogs, electronics, and plants." The latent from the model trained with MonoLoss targets ``goats," but also exhibits semantic drift in the lowest activations, capturing unrelated images such as ``trucks, monkeys, and insects."
}
\label{fig:qualitative_topk_rank1428}
\end{figure*}

\begin{figure*}[ht]
\centering
\includegraphics[width=\textwidth]{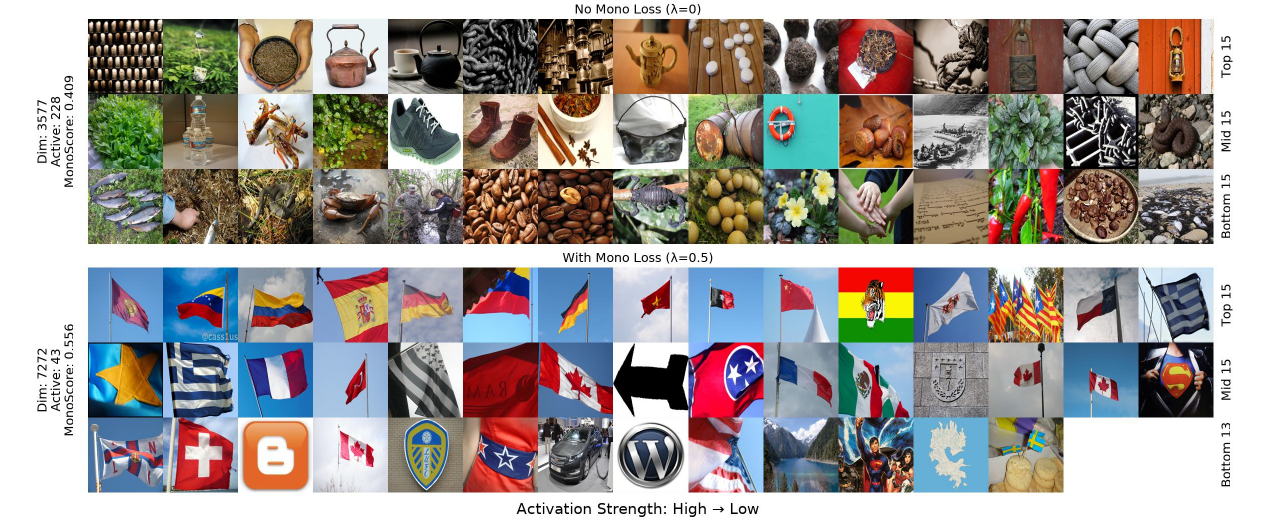}
\caption{
Comparison of activation patterns for TopK SAEs at rank 2361. The latent from the baseline (no MonoLoss) lacks a clear semantic concept, triggering on an incoherent mix of multiple concepts. In contrast, the latent from the model trained with MonoLoss targets ``flags," consistently identifying national flags and related ``logos and symbols" across most of its dynamic range.
}
\label{fig:qualitative_topk_rank2361}
\end{figure*}

\begin{figure*}[ht]
\centering
\includegraphics[width=\textwidth]{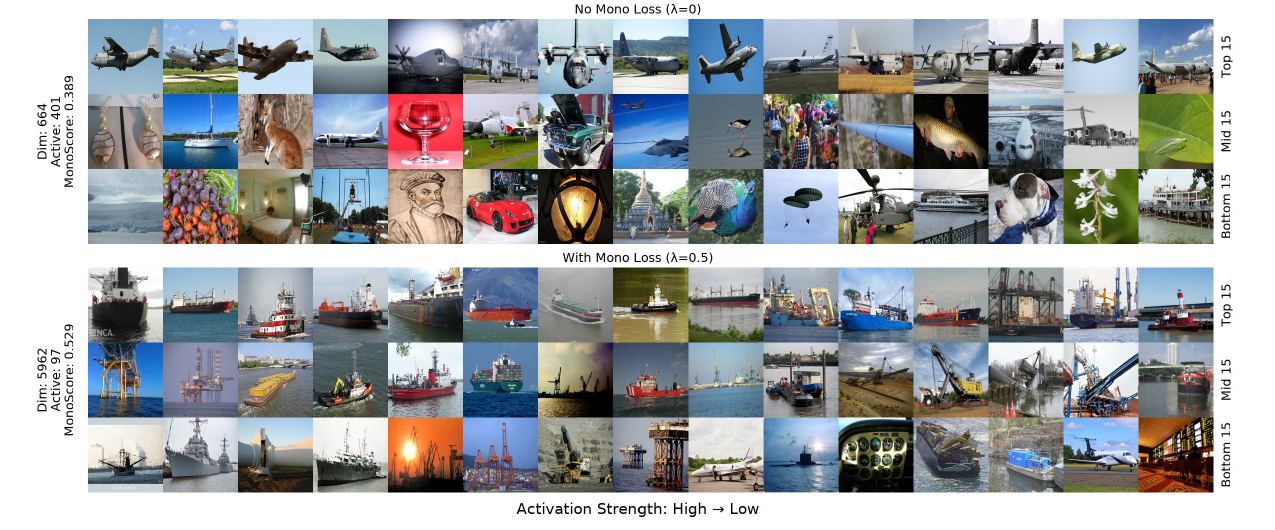}
\caption{
Comparison of activation patterns for TopK SAEs at rank 3215. The latent from the baseline (no MonoLoss) targets ``transport aircraft" but exhibits severe polysemantic drift in weaker activations, mixing in unrelated concepts like ``animals, jewelry, and landscapes." In contrast, the latent from the model trained with MonoLoss demonstrates stronger coherence, primarily targeting ``ships and maritime industry" across the majority of its dynamic range.
}
\label{fig:qualitative_topk_rank3215}
\end{figure*}

\begin{figure*}[ht]
\centering
\includegraphics[width=\textwidth]{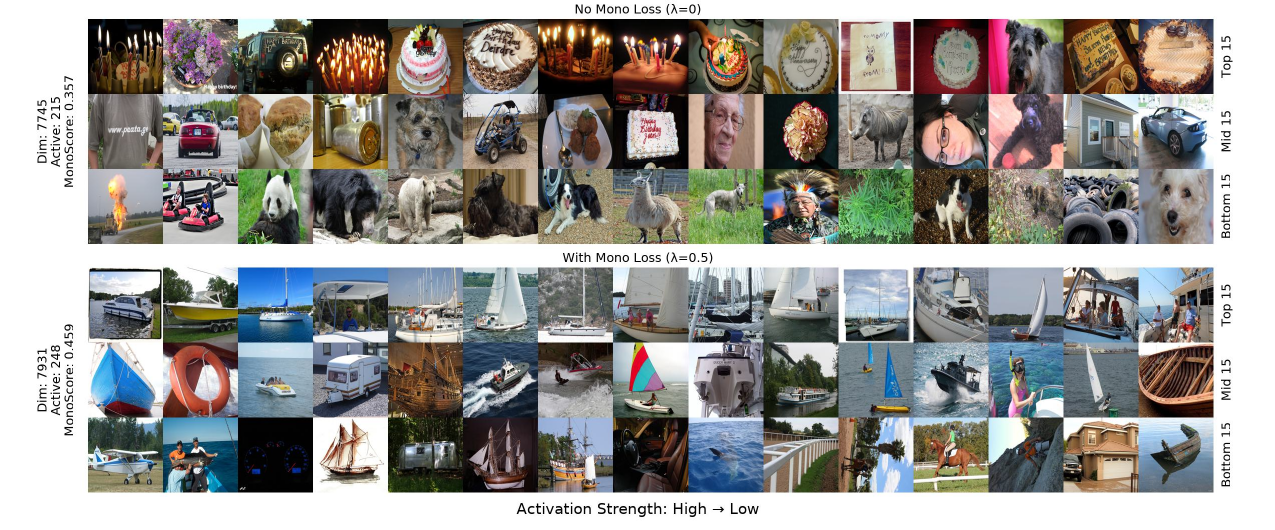}
\caption{
Comparison of activation patterns for TopK SAEs at rank 5324. The latent from the baseline (no MonoLoss) targets ``birthday cakes" but is highly polysemantic, mixing in unrelated images like ``dogs and cars." The latent from the model trained with MonoLoss targets ``boats and sailboats," but also exhibits some semantic drift in the lowest activations, capturing unrelated objects like ``horses and houses."
}
\label{fig:qualitative_topk_rank5324}
\end{figure*}

\begin{figure*}[ht]
\centering
\includegraphics[width=\textwidth]{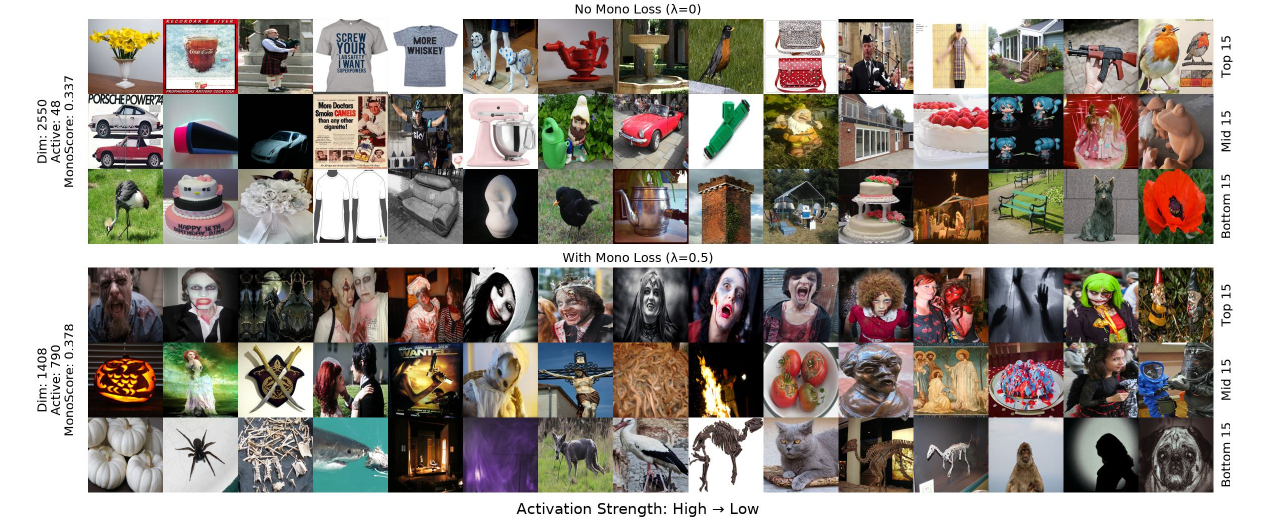}
\caption{
Comparison of activation patterns for TopK SAEs at rank 6678. The latent from the baseline (no MonoLoss) appears incoherent, triggering on a random assortment of ``clothing, cars, and furniture." In contrast, the latent from the model trained with MonoLoss targets ``horror and Halloween themes" (zombies, pumpkins), although it exhibits significant semantic drift in the weaker activations, capturing ``religious art and animals."}
\label{fig:qualitative_topk_rank6678}
\end{figure*}

\clearpage 

\section{Additional Monosemanticity Score Profiles}
\label{app:additional_ms_plots}


\noindent\begin{minipage}{\columnwidth}
\textbf{Monosemanticity across different $\lambda$ values.}
For a fixed encoder--SAE pair and a given MonoLoss coefficient $\lambda$, we compute a MonoScore for each of the $d = 8192$ latents. We then sort these scores in decreasing order and plot them as a function of a \emph{normalized latent index} in $[0,1]$: the left end of the curve corresponds to the highest-scoring (most monosemantic) latent and the right end to the lowest-scoring (least monosemantic) latent. Repeating this procedure for different values of $\lambda$ and overlaying the resulting curves gives a profile of how monosemanticity is distributed across latents as the strength of MonoLoss is varied. Examples for BatchTopK SAEs with CLIP, SigLIP2, and ViT encoders are shown in Figures~\ref{fig:latents_ms_batchtopk_clip}--\ref{fig:latents_ms_batchtopk_vit}.
\end{minipage}

\begin{figure}[ht]
    \centering
    \includegraphics[width=0.5\linewidth]{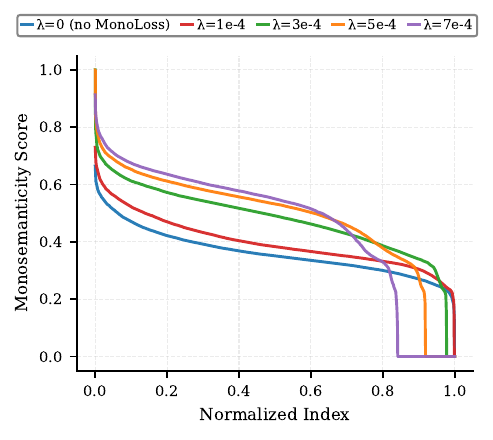}
    \caption{
        Effect of the monosemanticity loss coefficient ($\lambda$) on latent monosemanticity scores for BatchTopK SAE \cite{batchtopksae} ($k=64$) trained on CLIP image features. Latents are sorted by decreasing score. Increasing $\lambda$ improves monosemanticity across the distribution (upward shift) but reduces the number of interpretable features, as the lowest-ranking latents drop to zero score.
    }
    \label{fig:latents_ms_batchtopk_clip}
\end{figure}

\begin{figure}[ht]
    \centering
    \includegraphics[width=0.5\linewidth]{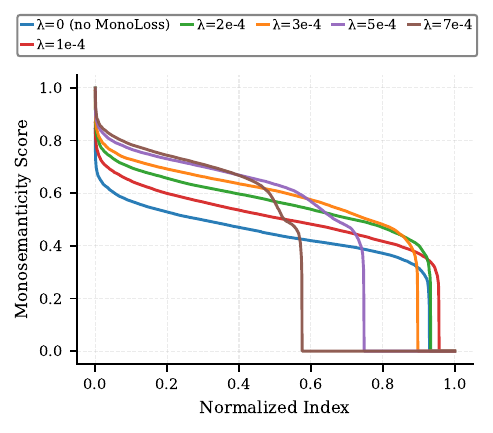}
    \caption{
        Effect of the monosemanticity loss coefficient ($\lambda$) on latent monosemanticity scores for BatchTopK SAE ($k=64$) trained on SigLIP2 image features. Latents are sorted by decreasing score. Higher $\lambda$ values yield significantly more monosemantic latents in the top half of the distribution. However, this comes at a steeper cost than in CLIP, with nearly 40\% of latents dropping to zero score under the strongest regularization ($\lambda=7\text{e-}4$).
}
    \label{fig:latents_ms_batchtopk_siglip2}
\end{figure}

\begin{figure}[H]
\centering
\includegraphics[width=0.5\linewidth]{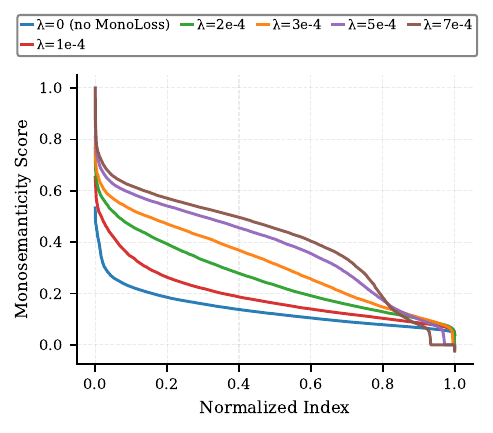}
\caption{
Effect of the monosemanticity loss coefficient ($\lambda$) on latent monosemanticity scores for BatchTopK SAE ($k=64$) trained on supervised ViT features. Latents are sorted by decreasing score. In this setting, MonoLoss produces the most dramatic relative improvement, lifting the entire distribution significantly above the baseline. Notably, this gain comes with minimal latent collapse, as the optimized latents remain active and interpretable across nearly the entire dictionary range.
}
\label{fig:latents_ms_batchtopk_vit}
\end{figure}

\section{Finetune experiment settings}
\label{app:Finetune_settings}

    \textbf{Datasets.} We employed ImageNet-1K, CIFAR-10, and CIFAR-100 for finetuning with standard splits. ImageNet-1K includes 1,281,167 training images and 50,000 validation images with 1,000 classes. CIFAR-10/100 consists of 50000 training images and 10000 test images with 10/100 categories. \\
    \noindent\textbf{Hyperparameters.} For CLIP-ViT-B/32, we adopted LoRA to efficiently finetune the model with a rank of 16 and alpha of 32 for the query and value projection matrices of all layers.    
    The details for experiments are as follow:
    \begin{table}[h]
        \centering
        \small
        \begin{tabular} {c cc cc}
            \toprule
            \multirow{2}{*}{\textbf{Setting}} &  \multicolumn{2}{c}{\textbf{ResNet-50}} & \multicolumn{2}{c}{\textbf{CLIP-ViT-B/32}} \\ 
            \cline{2-5}  & 
            ImangeNet-1K & CIFAR-10/100 & ImangeNet-1K & CIFAR-10/100 \\         
            \hline
            Learning rate & 1e-3           & 1e-1        & 1e-2             &      1e-2  \\
            Optimizer & SGD            & SGD         & SGD              & SGD            \\
            Epoch    & 90             & 90          & 90               & 90             \\
            Scheduler & Cosine         & Cosine      & Constant         & Constant       \\
            Batch size & 1024           & 1024        & 1024             & 1024           \\
            Warming epochs & 5              & 0           & 0                & 0              \\
            Weight decay  & 0.00002        & 0.00002     & 0.0001           & 0.0001   \\     
            \bottomrule
        \end{tabular}
        \caption{Training hyper-parameters.}
        \label{tab:training_config}
    \end{table}
    
    \noindent\textbf{Cosine similarity calculation.} In MonoScore, the metric requires the cosine similarity matrices and relevance matrices. To make the metric focus only on activation during training, we pre-computed the similarity matrices from the frozen CLIP-ViT-B/32 and retained them during training without updating.

\section{More qualitative results of finetune experiments}
\label{app:finetuning_qualitative}

    We studies more qualitative evaluations of finetuning ResNet-50 and CLIP-ViT-B/32 on ImageNet-1K, CIFAR-10, and CIFAR-10 (Figure \ref{fig:cifar10_rn50_10} - \ref{fig:cifar100_clip_vit_b_32_7}). In each figure, we select top-15, middle-15, and bot-15 of two dimensions with the same rank based on MonoScore from two settings (without MonoLoss and with MonoLoss). Overall, the qualitative results show that our proposed loss increases the homogeneity of concepts in the activating images.

    \begin{figure*}[ht]
        \centering
        \includegraphics[width=\textwidth]{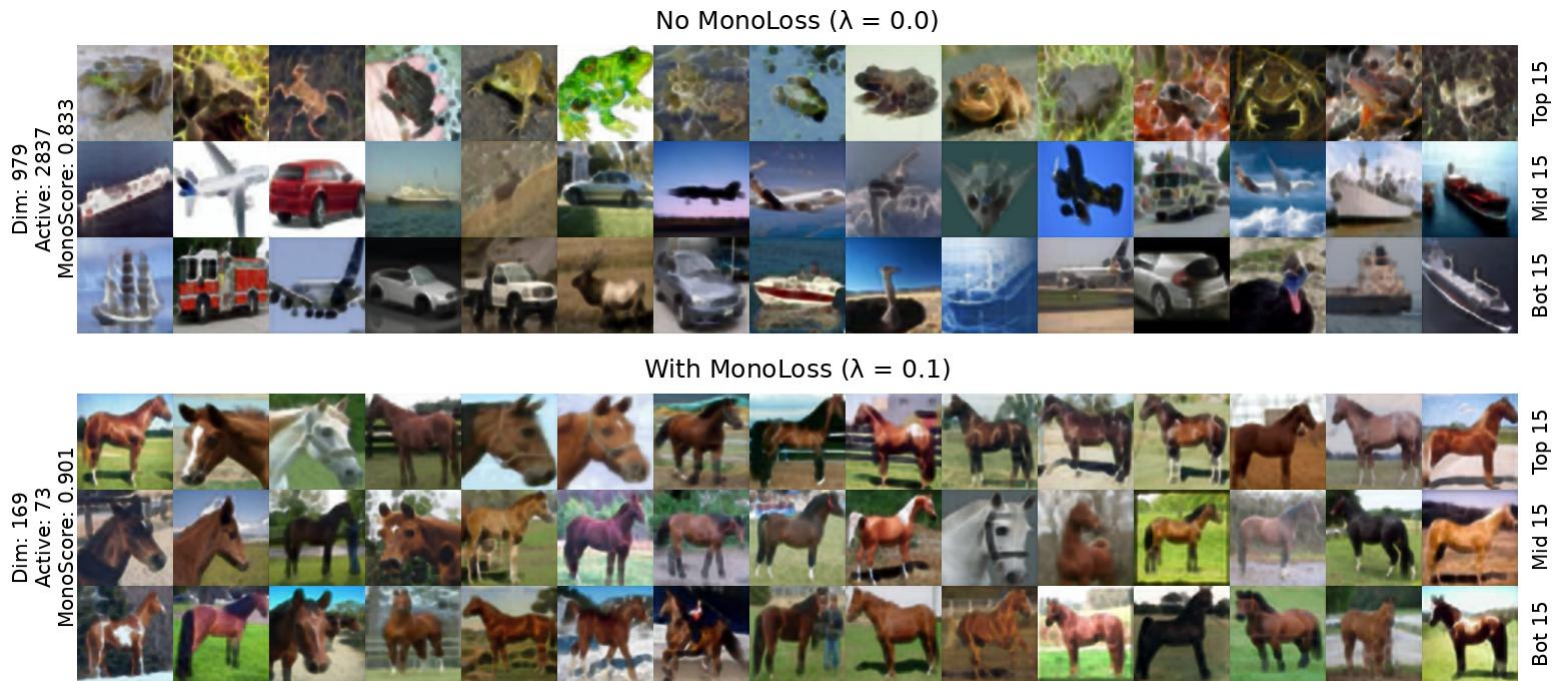}
        \caption{
        Comparison of activation patterns with and without monosemanticity loss for ResNet50 on CIFAR-10 from the same rank ($11$ out of $2048$). For each latent, we show top, middle, and bottom positively activated samples, ordered by activation strength from high to low. The top three rows finetuning without MonoLoss show diverse categories (frogs, ships, airplanes,...), whereas the bottom three rows finetuning with MonoLoss present a single concept of horses. Moreover, the number of activating images with MonoLoss is only 73, which means that the latent selectively activates for the corresponding concepts. Additional examples are in the Supplementary Material \ref{app:finetuning_qualitative}.
        }
        \label{fig:cifar10_rn50_10}
    \end{figure*}


    \begin{figure*}[ht]
        \centering
        \includegraphics[width=\textwidth]{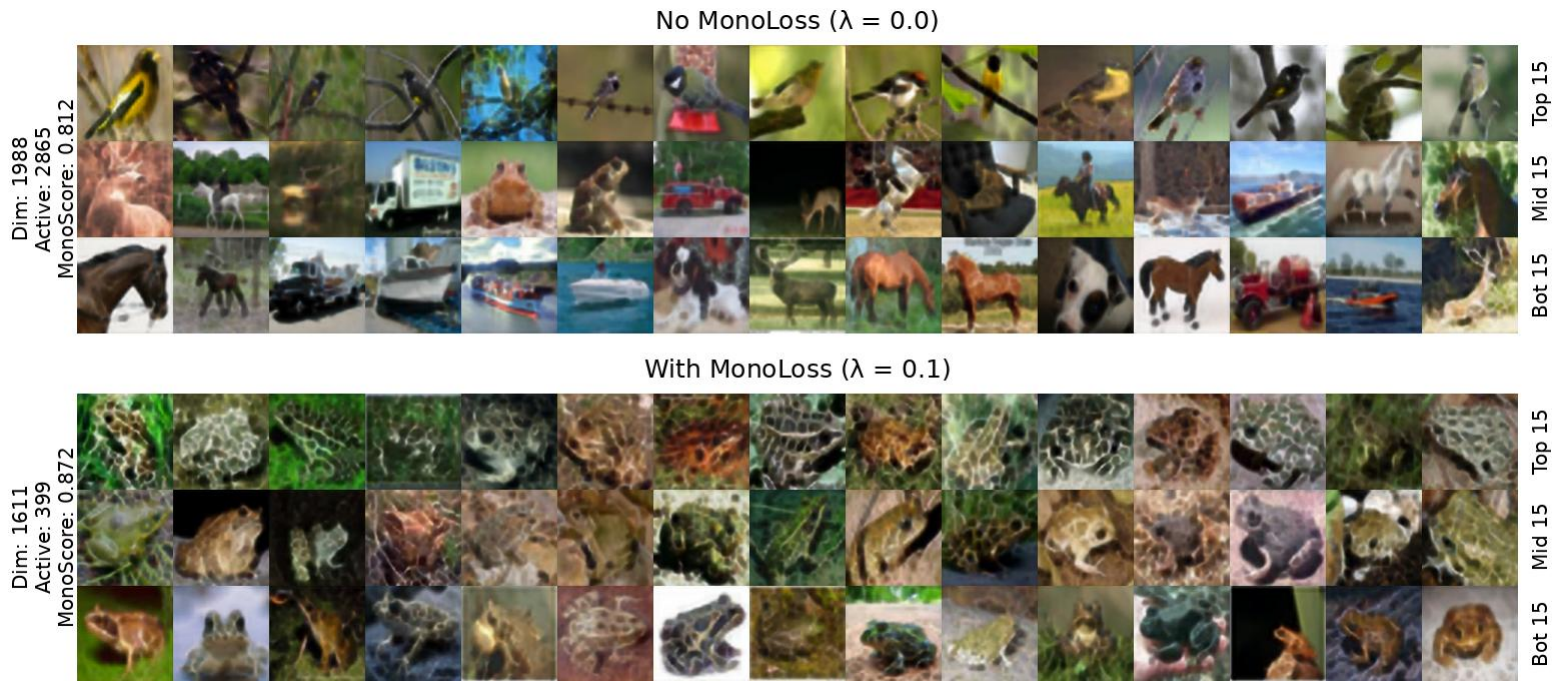}
        \caption{
        Comparison of activation patterns with and without monosemanticity loss for ResNet50 on CIFAR-10 from the same rank ($101$ out of $2048$). For each latent, we show top, middle, and bottom positively activated samples, ordered by activation strength from high to low. The top three rows finetuning without MonoLoss show diverse categories (birds, ships, horses,...), whereas the bottom three rows finetuning with MonoLoss present a single concept of frogs. 
        }
        \label{fig:cifar10_rn50_100}
    \end{figure*}

    \begin{figure*}[ht]
        \centering
        \includegraphics[width=\textwidth]{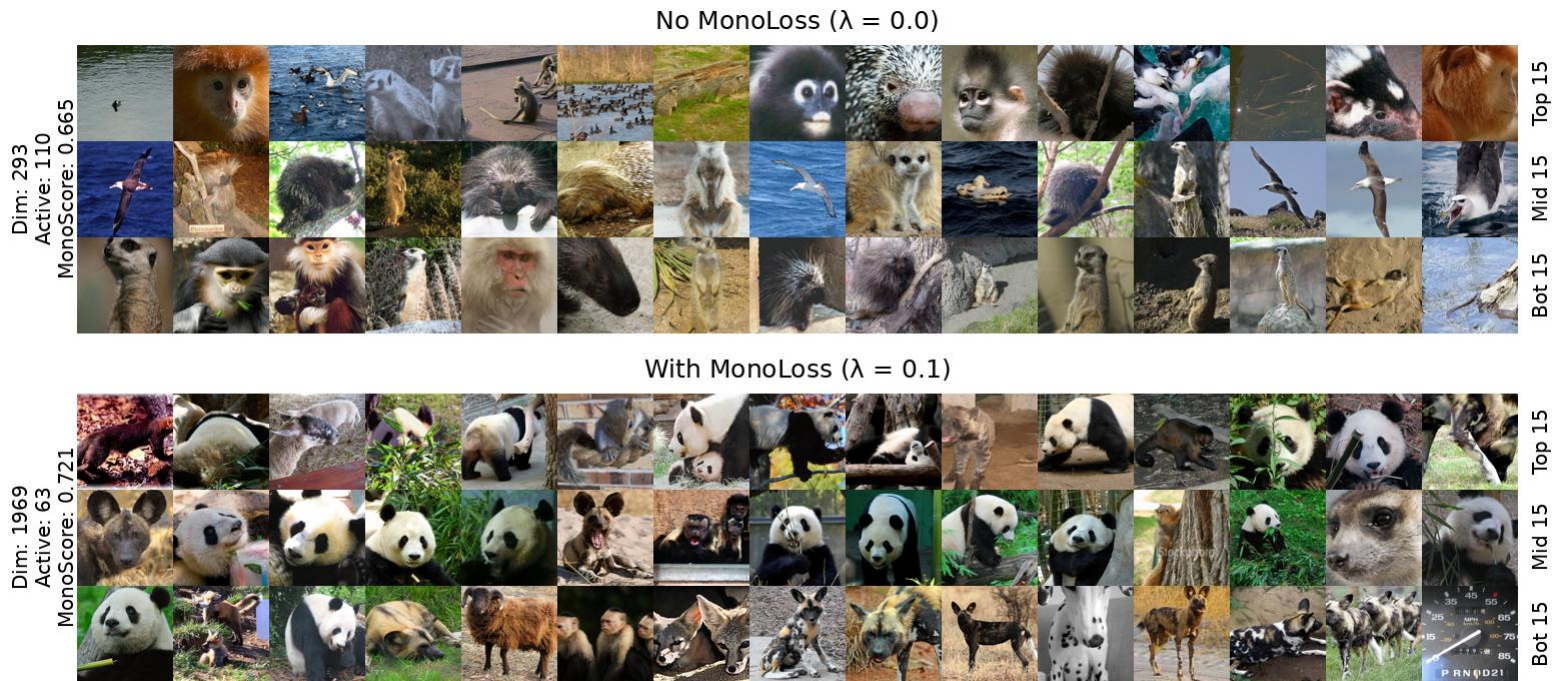}
        \caption{
        Comparison of activation patterns with and without monosemanticity loss for ResNet50 on ImageNet-1K from the same rank ($21$ out of $2048$). For each latent, we show top, middle, and bottom positively activated samples, ordered by activation strength from high to low. The top three rows finetuning without MonoLoss show much diverse categories (mix of different kind of animals), whereas the bottom three rows finetuning with MonoLoss present mostly pandas.
        }
        \label{fig:imagenet_rn50_20}
    \end{figure*}

    \begin{figure*}[ht]
        \centering
        \includegraphics[width=\textwidth]{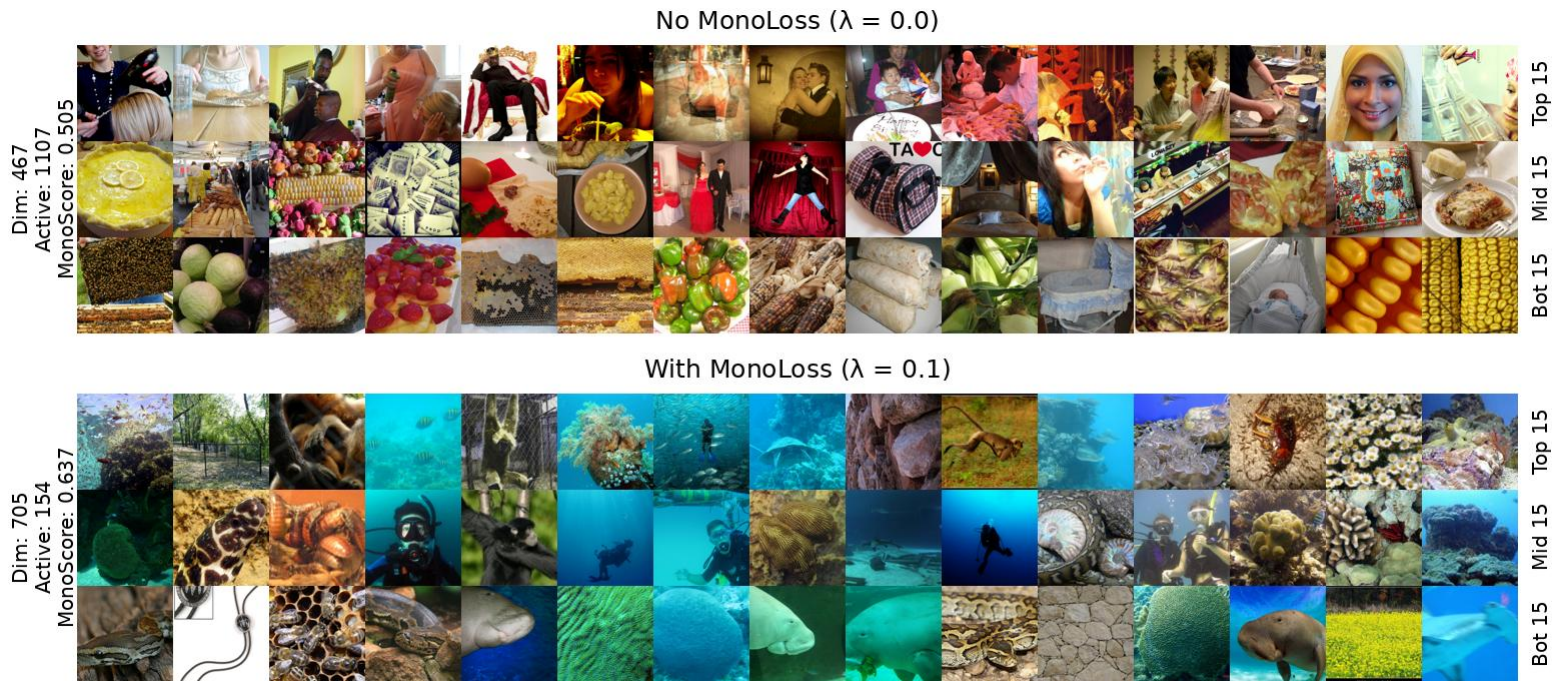}
        \caption{
        Comparison of activation patterns with and without monosemanticity loss for ResNet50 on ImageNet-1K from the same rank ($61$ out of $2048$). For each latent, we show top, middle, and bottom positively activated samples, ordered by activation strength from high to low. The top three rows finetuning without MonoLoss show much diverse categories (mix of human, objects, and vegetables), whereas the bottom three rows finetuning with MonoLoss present mostly concepts related to ocean.
        }
        \label{fig:imagenet_rn50_60}
    \end{figure*}

    \begin{figure*}[ht]
        \centering
        \includegraphics[width=\textwidth]{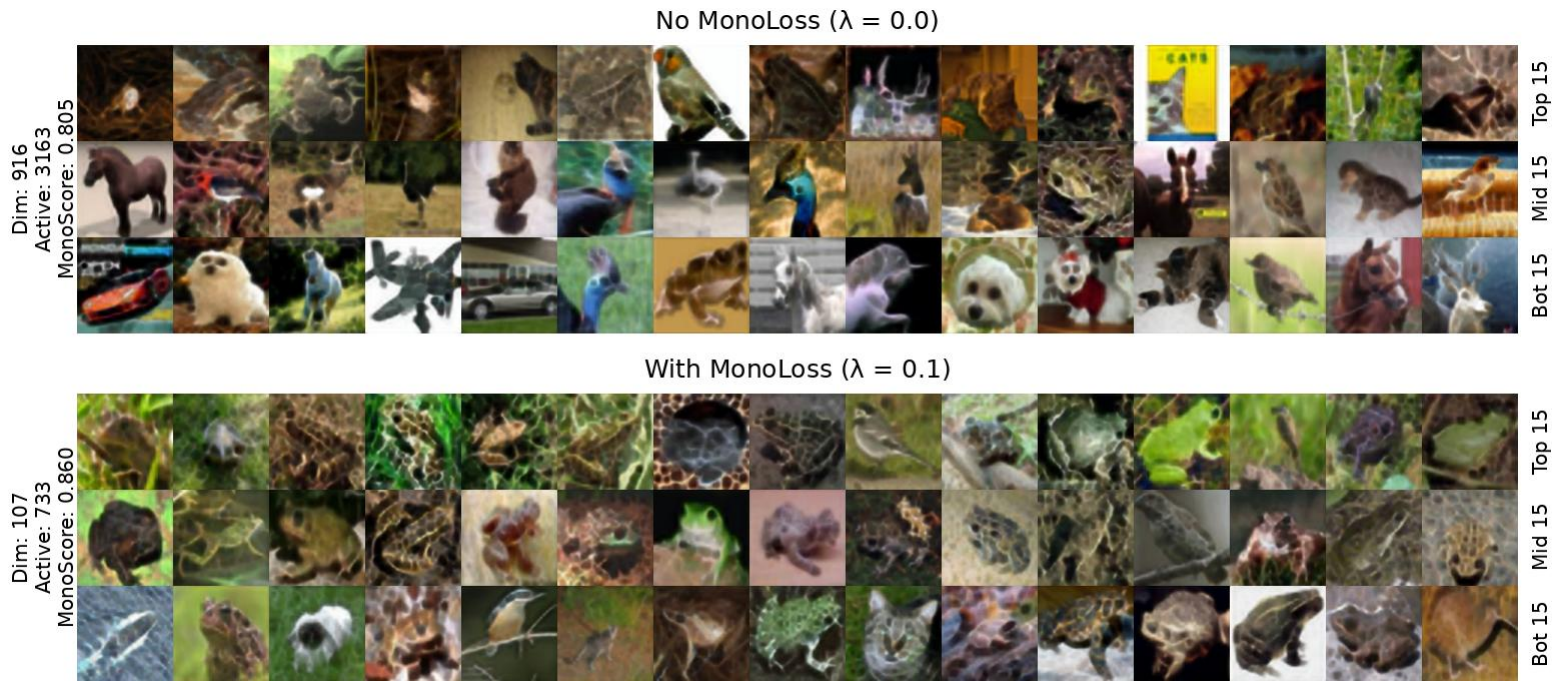}
        \caption{
        Comparison of activation patterns with and without monosemanticity loss for ResNet50 on CIFAR-100 from the same rank ($11$ out of $2048$). For each latent, we show top, middle, and bottom positively activated samples, ordered by activation strength from high to low. The top three rows finetuning without MonoLoss show much diverse categories (mix of cars, airplanes, dogs, cats,...), whereas the bottom three rows finetuning with MonoLoss present only frogs.
        }
        \label{fig:cifar100_rn50_10}
    \end{figure*}

    \begin{figure*}[ht]
        \centering
        \includegraphics[width=\textwidth]{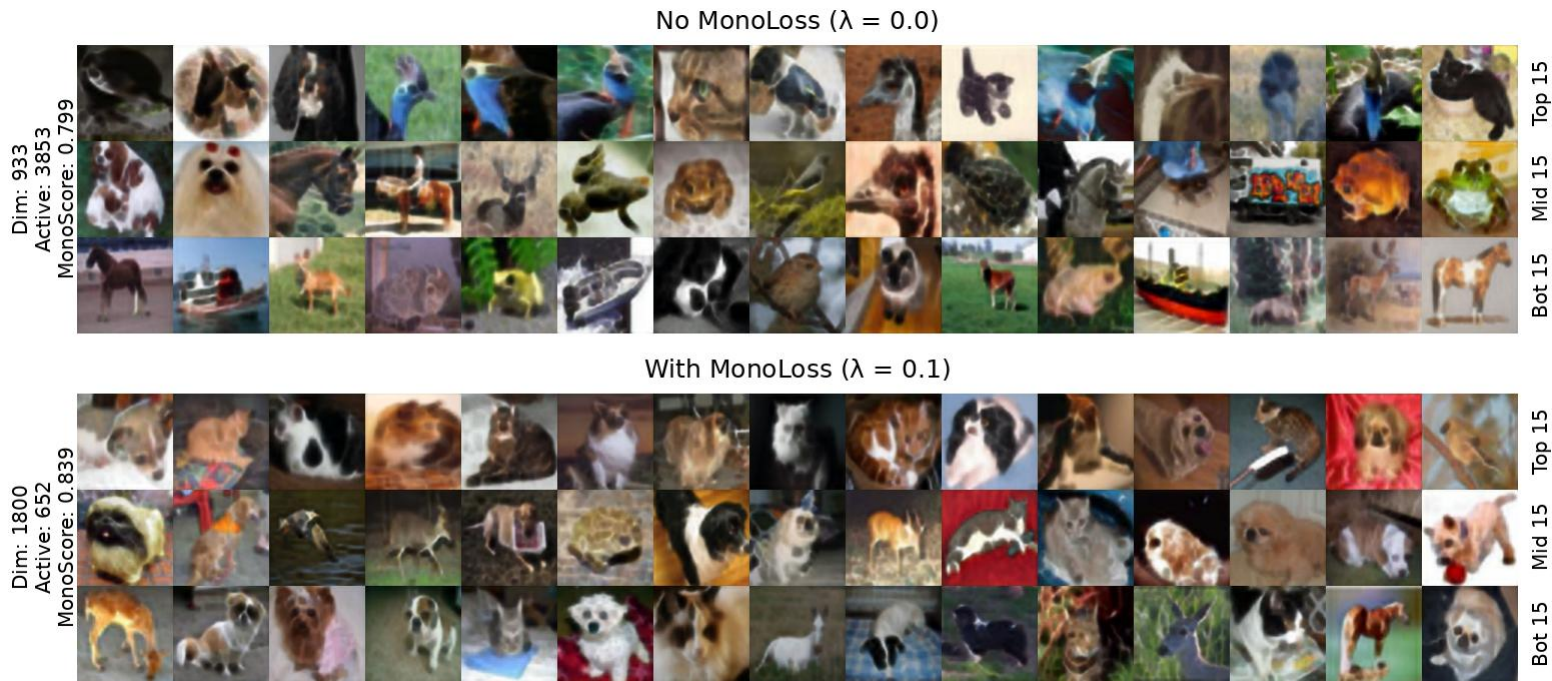}
        \caption{
        Comparison of activation patterns with and without monosemanticity loss for ResNet50 on CIFAR-100 from the same rank ($31$ out of $2048$). For each latent, we show top, middle, and bottom positively activated samples, ordered by activation strength from high to low. The top three rows finetuning without MonoLoss show much diverse categories (mix of birds, cats, dogs, frogs,...), whereas the bottom three rows finetuning with MonoLoss present mostly dogs.
        }
        \label{fig:cifar100_rn50_30}
    \end{figure*}

    \begin{figure*}[ht]
        \centering
        \includegraphics[width=\textwidth]{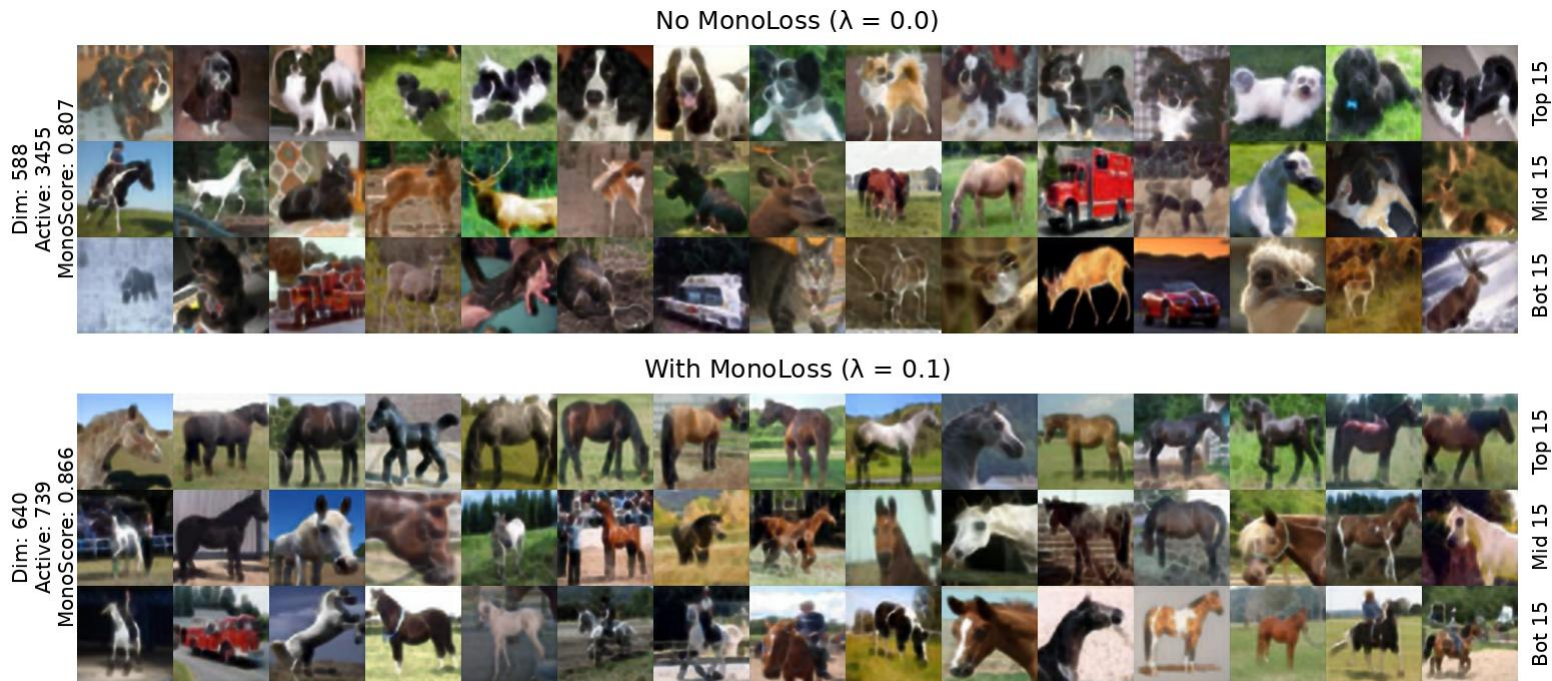}
        \caption{
        Comparison of activation patterns with and without monosemanticity loss for CLIP-ViT-B/32 on CIFAR-10 from the same rank ($21$ out of $768$). For each latent, we show top, middle, and bottom positively activated samples, ordered by activation strength from high to low. The top three rows finetuning without MonoLoss show much diverse categories (mix of dogs, horses, cars, deers,...), whereas the bottom three rows finetuning with MonoLoss present mostly horses.
        }
        \label{fig:cifar10_clip_vit_b_32_20}
    \end{figure*}

    \begin{figure*}[ht]
        \centering
        \includegraphics[width=\textwidth]{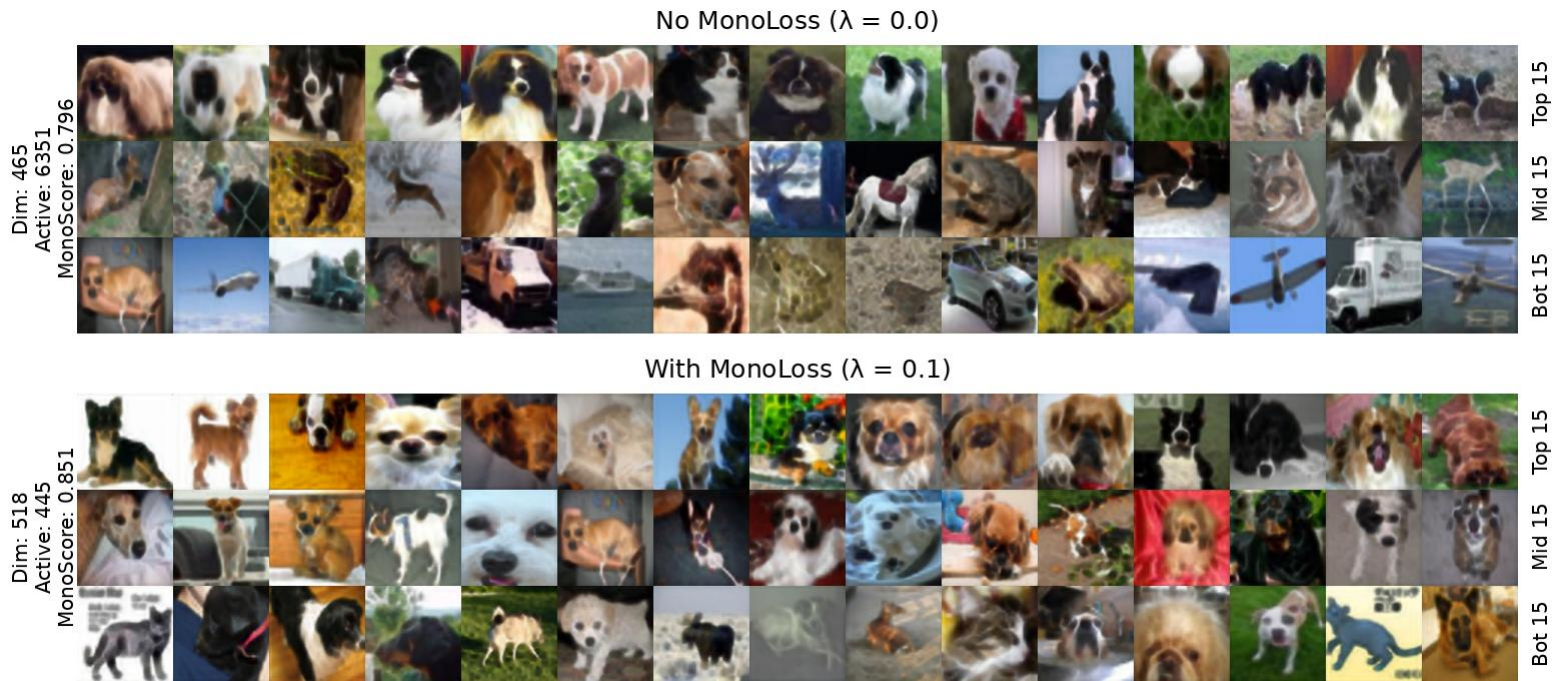}
        \caption{
        Comparison of activation patterns with and without monosemanticity loss for CLIP-ViT-B/32 on CIFAR-10 from the same rank ($61$ out of $768$). For each latent, we show top, middle, and bottom positively activated samples, ordered by activation strength from high to low. The top three rows finetuning without MonoLoss show much diverse categories (mix of dogs, cats, cars, birds,...), whereas the bottom three rows finetuning with MonoLoss present mostly dogs.
        }
        \label{fig:cifar10_clip_vit_b_32_60}
    \end{figure*}

    \begin{figure*}[ht]
        \centering
        \includegraphics[width=\textwidth]{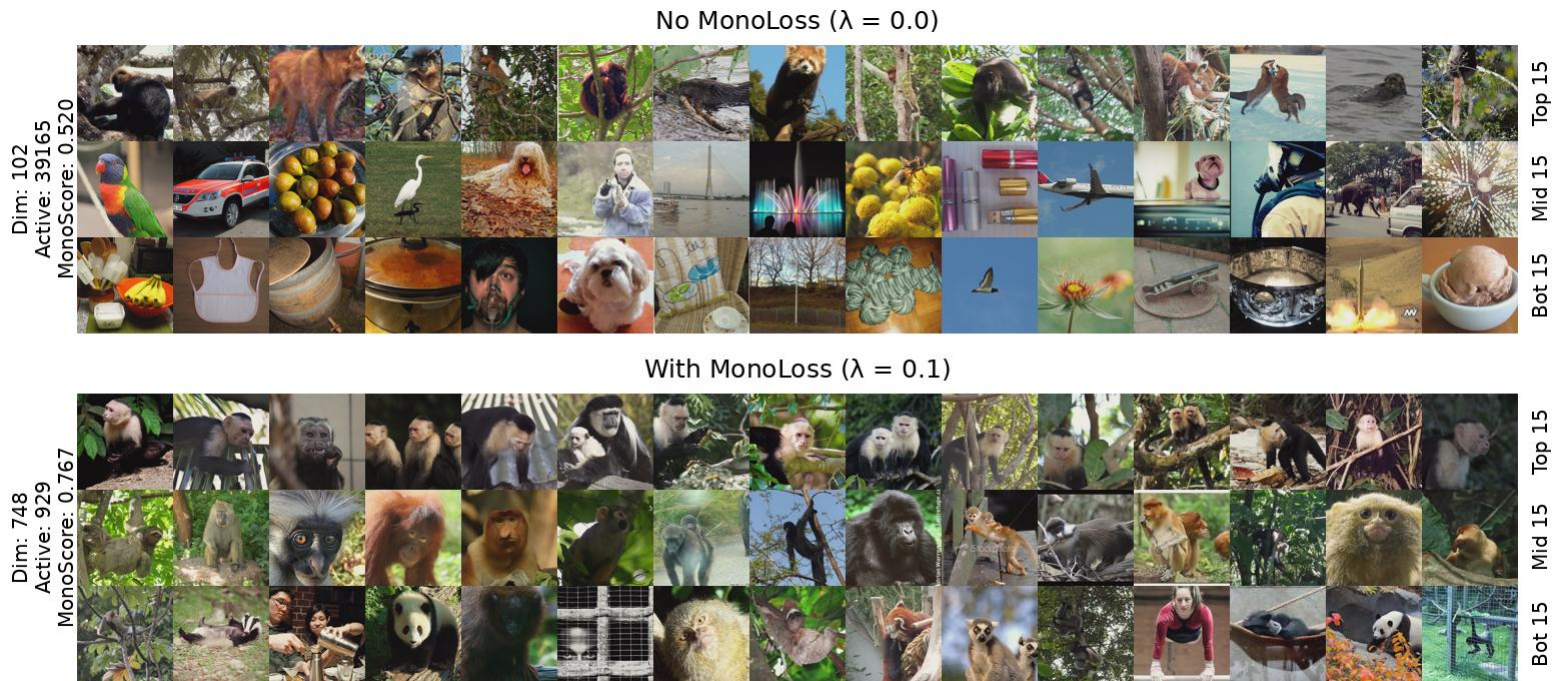}
        \caption{
        Comparison of activation patterns with and without monosemanticity loss for CLIP-ViT-B/32 on ImageNet-1K from the same rank ($1$ out of $768$). For each latent, we show top, middle, and bottom positively activated samples, ordered by activation strength from high to low. The top three rows finetuning without MonoLoss show much diverse categories (mix of objects, cars, human, animals...), whereas the bottom three rows finetuning with MonoLoss present mostly primates.
        }
        \label{fig:imagenet_clip_vit_b_32_0}
    \end{figure*}

    \begin{figure*}[ht]
        \centering
        \includegraphics[width=\textwidth]{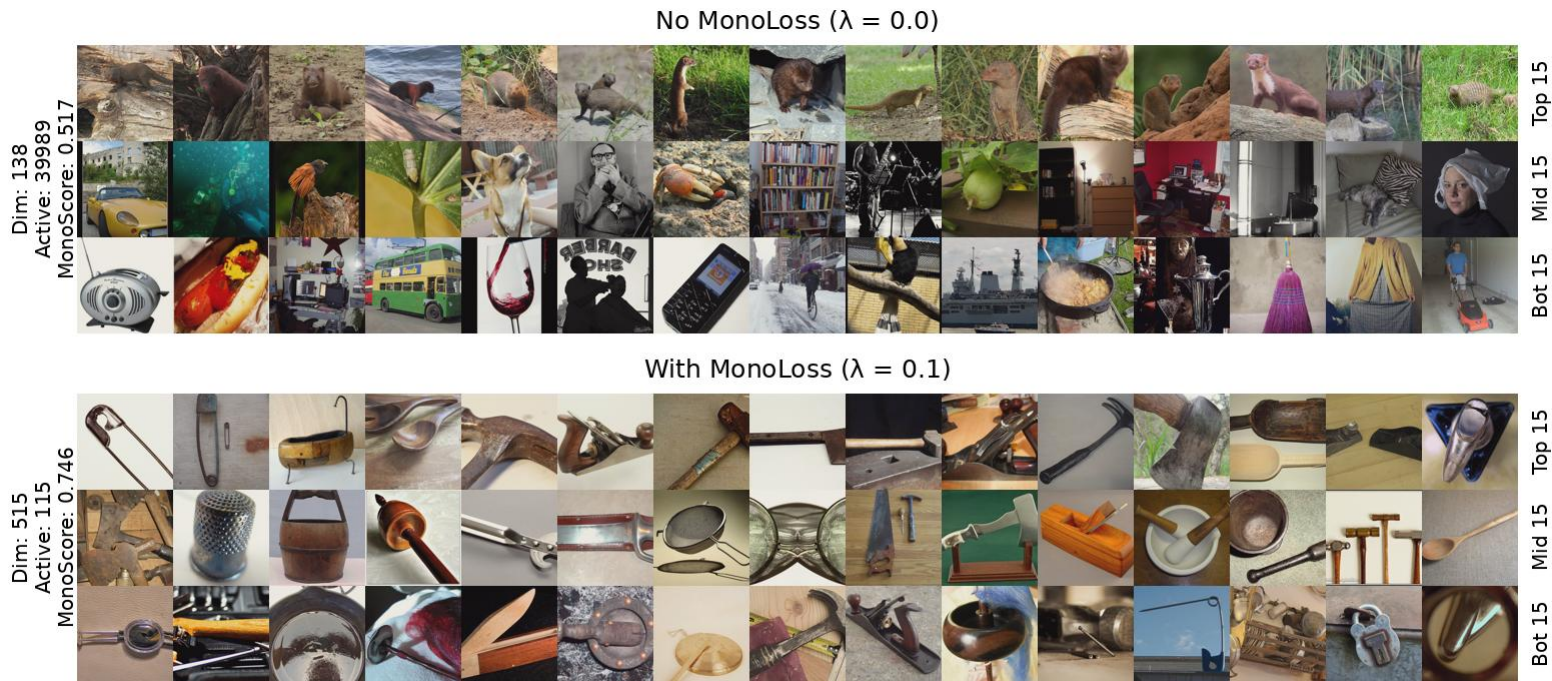}
        \caption{
        Comparison of activation patterns with and without monosemanticity loss for CLIP-ViT-B/32 on ImageNet-1K from the same rank ($2$ out of $768$). For each latent, we show top, middle, and bottom positively activated samples, ordered by activation strength from high to low. The top three rows finetuning without MonoLoss show much diverse categories (mix of animals, tools, human, dog,...), whereas the bottom three rows finetuning with MonoLoss present mostly tools.
        }
        \label{fig:imagenet_clip_vit_b_32_1}
    \end{figure*}

    \begin{figure*}[ht]
        \centering
        \includegraphics[width=\textwidth]{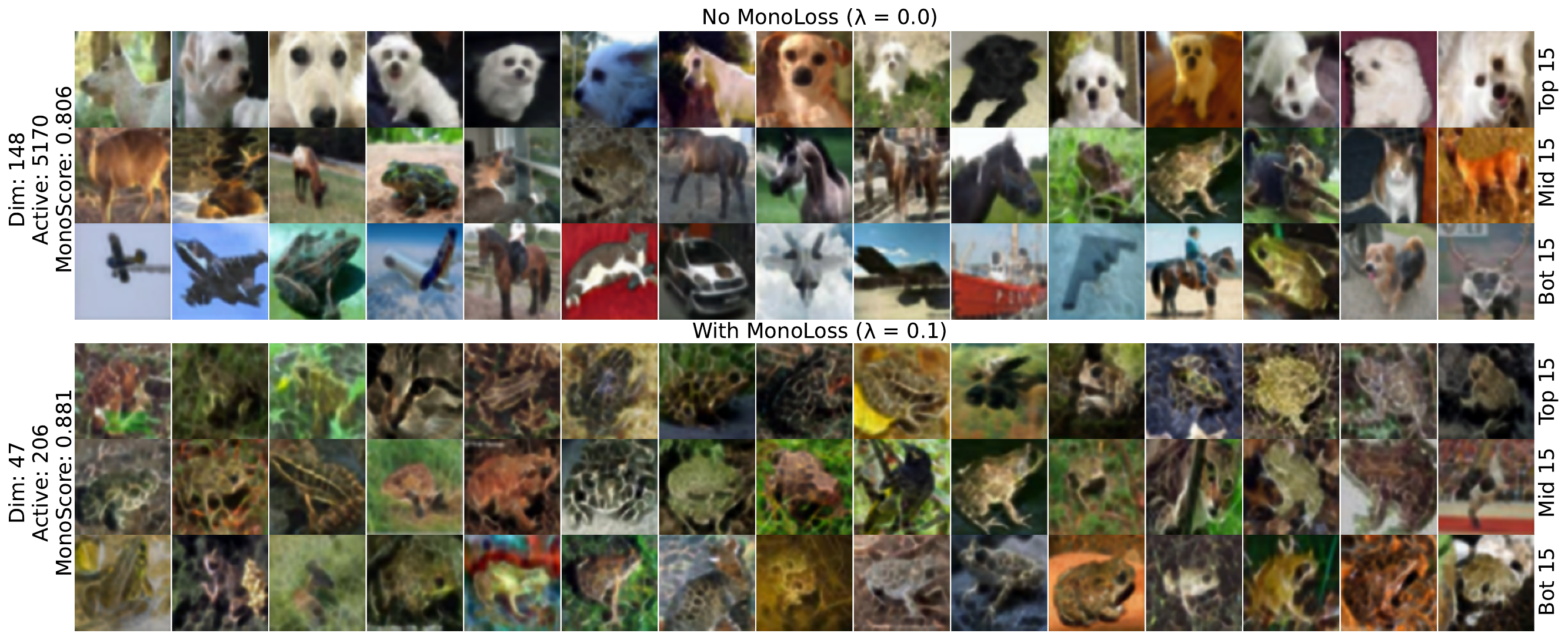}
        \caption{
        Comparison of activation patterns with and without monosemanticity loss for CLIP-ViT-B/32 on ImageNet-1K from the same rank ($2$ out of $768$). For each latent, we show top, middle, and bottom positively activated samples, ordered by activation strength from high to low. The top three rows finetuning without MonoLoss show much diverse categories (mix of dogs, airplanes, horses,...), whereas the bottom three rows finetuning with MonoLoss present frogs.
        }
        \label{fig:cifar100_clip_vit_b_32_1}
    \end{figure*}

    \begin{figure*}[ht]
        \centering
        \includegraphics[width=\textwidth]{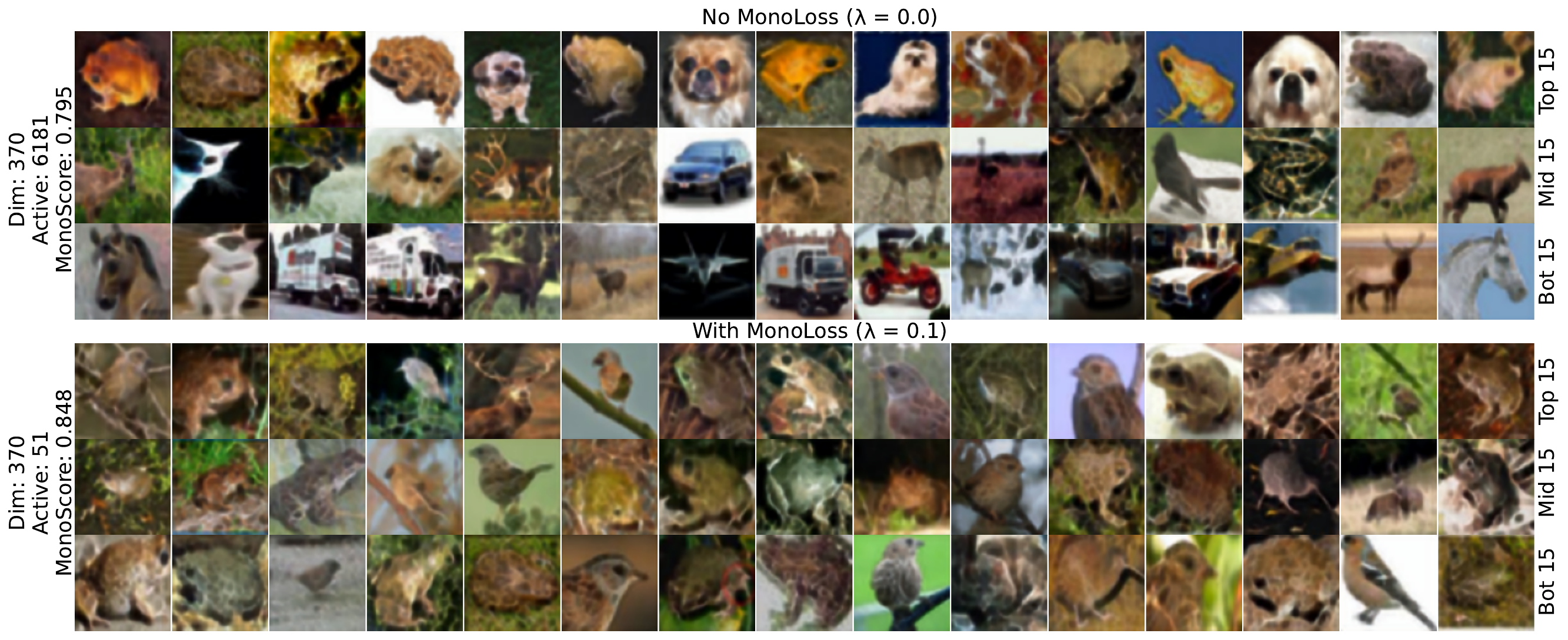}
        \caption{
        Comparison of activation patterns with and without monosemanticity loss for CLIP-ViT-B/32 on ImageNet-1K from the same rank ($8$ out of $768$). For each latent, we show top, middle, and bottom positively activated samples, ordered by activation strength from high to low. The top three rows finetuning without MonoLoss show much diverse categories (mix of frogs, dogs, deers, cars,...), whereas the bottom three rows finetuning with MonoLoss present mostly brown animals.
        }
        \label{fig:cifar100_clip_vit_b_32_7}
    \end{figure*}
    
\end{document}